\documentclass[10pt,journal,compsoc]{IEEEtran}
\ifCLASSOPTIONcompsoc
  \usepackage[nocompress]{cite}
\else
  \usepackage{cite}
\fi

\ifCLASSINFOpdf
\else
\fi
\usepackage{amsmath,mathrsfs,amssymb}
\usepackage{hyperref}
\usepackage{booktabs}
\usepackage{xcolor}
\usepackage{array} 
\usepackage{ragged2e} 
\usepackage{mathtools}
\usepackage{booktabs}       
\usepackage{longtable}
\usepackage{threeparttable}
\usepackage{siunitx} 
\usepackage{caption}
\usepackage{multirow}
\usepackage{multicol}
\usepackage{graphicx}
\usepackage{bm}

\usepackage{amsthm}
\usepackage{dsfont}
\usepackage{lipsum} 

\usepackage{makecell,rotating,multirow,diagbox}
\usepackage{bm}
\usepackage{makecell}
\usepackage{subcaption}
\usepackage{cuted}

\usepackage{amssymb}
\usepackage{pifont}
\newcommand{\cmark}{\ding{51}}%
\newcommand{\xmark}{\ding{55}}%

\newtheorem{theorem}{Theorem}

\newtheorem{definition}{Definition}
\newtheorem{corollary}[theorem]{Corollary}

\def\a{\mathbf{a}}

\def\q{\mathbf{q}}
\def\z{\mathbf{z}}

\def\x{\bm{x}}

\def\y{\mathbf{y}}
\def\W{\mathbf{W}}

\def\w{\textbf{w}}

\usepackage{colortbl}
\usepackage{makecell}
\usepackage{multirow}
\usepackage{multicol}
\definecolor{LightCyan}{rgb}{0.6,0.6,0.5}
\definecolor{Gray}{gray}{0.85}
\definecolor{Gray2}{gray}{0.75}
\definecolor{Gray3}{gray}{0.65}
\definecolor{Gray4}{gray}{0.55}
\definecolor{Gray5}{gray}{0.45}
\definecolor{orange01}{HTML}{FC8D59} 
\definecolor{blue01}{HTML}{92BFDB}

\def\w{\boldsymbol{W}}
\def\a{\mathbf{a}}
\def\d{\mathbf{d}}

\newcommand{\rowstyle}[1]{\gdef\currentrowstyle{#1}#1\ignorespaces}
\newcolumntype{+}{>{\global\let\currentrowstyle\relax}}
\newcolumntype{^}{>{\currentrowstyle}}

\def\a{\mathbf{a}}

\def\x{\boldsymbol{x}}
\def\y{\boldsymbol{y}}

\def\a{\mathbf{a}}

\def\b{\mathbf{b}}
\def\z{\mathbf{z}}
\def\u{\boldsymbol{u}}
\def\w{\mathbf{W}}

\def\X{\mathbf{X}}
\def\q{\boldsymbol{q}}

\def\a{\mathbf{a}}

\def\a{\mathbf{a}}

\def\x{\boldsymbol{x}}

\def\X{\boldsymbol{X}}

\def\0{\boldsymbol{0}}

\def\x{\boldsymbol{x}}
\def\y{\boldsymbol{y}}
\def\z{\boldsymbol{z}}

\def\a{\boldsymbol{a}}

\def\w{\boldsymbol{w}}

\def\X{\boldsymbol{X}}

\def\W{\boldsymbol{W}}

\begin{document}
%
\title{Hyper-Compression: Model Compression via Hyperfunction}

%
%
%
%

\author{Feng-Lei Fan$^\dag$, Juntong Fan$^\dag$, Dayang Wang$^\dag$, Jingbo Zhang, Zelin Dong, Shijun Zhang, Ge Wang, Tieyong Zeng$^*$

\IEEEcompsocitemizethanks{
\IEEEcompsocthanksitem Feng-Lei Fan, Juntong Fan, and Dayang Wang are co-first authors.

\IEEEcompsocthanksitem Tieyong Zeng (zeng@math.cuhk.edu.hk) is the corresponding author.

\IEEEcompsocthanksitem Feng-Lei Fan, Juntong Fan, Dayang Wang, and Jingbo Zhang are with Department of Data Science, City University of Hong Kong, China SAR. 
\IEEEcompsocthanksitem 
Shijun Zhang is with Department of Applied Mathematics, The Hong Kong Polytechnic University, China SAR.
\IEEEcompsocthanksitem 
Ge Wang is with Department of Biomedical Engineering, Rensselaer Polytechnic Institute, NY, US.
\IEEEcompsocthanksitem 
Tieyong Zeng is with Department of Mathematics, The Chinese University of Hong Kong, China SAR.
}
}

\markboth{Journal of \LaTeX\ Class Files,~Vol.~14, No.~8, April~2025}%
{Shell \MakeLowercase{\textit{et al.}}: Bare Demo of IEEEtran.cls for Computer Society Journals}
%



\IEEEtitleabstractindextext{%
\begin{abstract}
The rapid growth of large models' size has far outpaced that of computing resources. To bridge this gap, encouraged by the parsimonious relationship between genotype and phenotype in the brain's growth and development, we propose the so-called Hyper-Compression that turns the model compression into the issue of parameter representation via a hyperfunction. Specifically, it is known that the trajectory of some low-dimensional dynamic systems can fill the high-dimensional space eventually. Thus, Hyper-Compression, using these dynamic systems as the hyperfunctions, represents the parameters of the target network by their corresponding composition number or trajectory length. This suggests a novel mechanism for model compression, substantially different from the existing pruning, quantization, distillation, and decomposition. Along this direction, we methodologically identify a suitable dynamic system with the irrational winding as the hyperfunction and theoretically derive its associated error bound. Next, guided by our theoretical insights, we propose several engineering twists to make the Hyper-Compression pragmatic and effective. Lastly, systematic and comprehensive experiments on \textcolor{black}{NLP models such as LLaMA and Qwen series and vision models} confirm that Hyper-Compression enjoys the following \textbf{PNAS} merits: 1) \textbf{P}referable compression ratio; 2) \textbf{N}o post-hoc retraining; 3) \textbf{A}ffordable inference time; and 4) \textbf{S}hort compression time. It compresses LLaMA2-7B in an hour and achieves close-to-int4-quantization performance, without retraining and with a performance drop of less than 1\%. We have open-sourced our code in \url{https://github.com/Juntongkuki/Hyper-Compression.git} for free download and evaluation.

\end{abstract}

}

\maketitle

\IEEEdisplaynontitleabstractindextext

%
\IEEEpeerreviewmaketitle

\IEEEraisesectionheading{\section{Introduction}\label{sec:introduction}}
\vspace{-0.2cm}

%
%
%
%

\IEEEPARstart{R}{ecently}, due to the pursuit of the scaling law, the escalating demand for computational resources by large models has presented formidable challenges for their deployment in resource-constrained environments \cite{ding2023parameter}. For example, a GPT-3 model has 175 billion parameters with a size of approximately 700 GB, while one of the most advanced GPUs (NVIDIA H100) has a memory capacity of only up to 80 GB. Serving large models well has become a strong technological imperative. To this end, currently, one prevalent way is to develop effective model compression approaches that crop the size of models while maintaining the performance.

Model compression \cite{zhu2023survey} predominantly revolves around four different classes of algorithms: pruning, quantization, knowledge distillation, and low-rank decomposition. All these algorithms were already proposed years or even decades ago. Though these techniques play a critical role in model compression, they face intrinsic challenges in the era of large models that are hard to overcome. First, the compression efficacy of some methods such as pruning and quantization is challenging to scale. A plethora of studies showed that pruning is submitted by the suboptimal compression rate \cite{zhu2023survey}, usually $2-4\times$. As for quantization, even the best it can do (1-bit quantization) remains unexciting because the compression ratio is bounded by a constant, which is handicapped to address the enlarged gap between models and the hardware resource in the era of large models. Moreover, when compressing a model at a large rate, the model's output is usually severely distorted. Thus, retraining and recalibration are mandated to recover the model's performance, which adds extra costs for curating data and computational resources. This not only is unfriendly for the industry that often favors agile deployment but also may introduce bias in the model. Observing these fundamental limits, we ask \textit{can we have a more promising technology roadmap, instead of keeping modifying the existing methods?}
The answer is affirmative. Here, we introduce Hyper-Compression, a novel and general-purpose approach that redefines model compression as a problem of parsimonious parameter representation. This concept stems from hyper-networks that are smaller networks but can generate weights for a significantly larger target network \cite{stanley2009hypercube, chauhan2024brief}, in analogy to the famous genomic bottleneck that a comparatively small genome can control the growth and development of a complicated brain \cite{shuvaev2024encoding}. We generalize the concept of hypernets into a ‘hyperfunction' (HF) and hypothesize that weights of a large network can be encoded by a function with few parameters. Mathematically, we use a parametric function $w_n=h(\theta;n)$ to encode the relationship between locations and weights of the target network, where $w_n$ is the $n$-th parameter of the target network, and $\theta$ collects the parameters of $h$. Should the memory footprint of $\theta$ be significantly less than that of $\{w_n\}_{n=1}^N$, $\theta$ can be stored on devices as a kind of compression for network parameters. During inference, $\theta$ is freely used to layer-by-layer recover weights of the original network via $h$.
\begin{figure}[h]  
\centering
\includegraphics[width=1\linewidth]{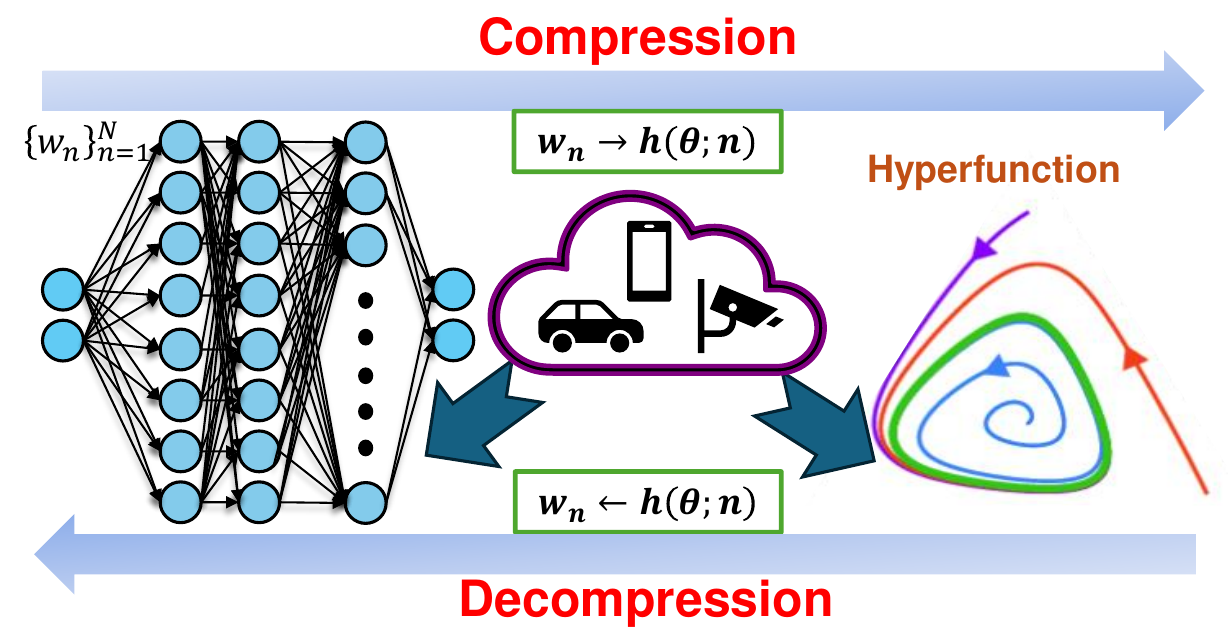}
\caption{We use a parametric function $w_n=h(\theta;n)$ to encode the relationship between locations and weights of the target network, where $w_n$ is the $n$-th parameter of the target network. $h$ is implied by the ergodic theory. When inference, parameters of the original networks are restored layer-by-layer.} 
\label{Hyper-Compression}
\end{figure}

Hereafter, we refer to model compression via the hyperfunction as the Hyper-Compression. Along this direction, we reasonably underscore that finding a suitable hyperfunction is an open-ended question. However, unexpectedly, we propose to leverage the density of trajectory in ergodic theory \cite{cornfeld2012ergodic} to construct the hyperfunction $h$. One of the most important findings in ergodic theory is that when the transformation $T$ is ergodic, a low-dimensional dynamic system can eventually visit a high-dimensional space given sufficient time. Hereafter, we refer to this property as the trajectory density. Informally, there exists $\x_0$ such that $\forall \x$ and $\epsilon>0$, $\exists k$ such that
\begin{equation}
\|\x -\underbrace{(T\circ T\circ\cdots\circ T)}_{k}(\x_0)\|=\|\x - T^k(\x_0)\|<\epsilon.
\end{equation}
Thus, the hyperfunction $h=T^k(\cdot)$ and we can “compress" a high-dimensional vector $\x$ into a number $k$, which should induce a good compression rate. Next, as the initial step, we use the irrational winding \cite{katok1995introduction} as the suitable dynamic system to validate the feasibility of the proposed framework for model compression.

Then, we theoretically derive the error bound of the Hyper-Compression based on the irrational winding, thereby characterizing the key ingredients that affect the error. Next, guided by theoretical insight, we propose several pragmatic engineering twists such as scaling to cope with outliers appropriately, adjusting irrational direction to diminish the error, and applying the KD-Tree \cite{zhou2008real} for fast compression. These twists lead to an effective and efficient model compression algorithm. Specifically, it offers four distinct characters, succinctly summarized as \textbf{PNAS}: 1) \textbf{P}referable compression ratio; 2) \textbf{N}o post-hoc retraining\&recalibration; 3) \textbf{A}ffordable inference time; and 4) \textbf{S}hort compression time. Our work can facilitate the harmony between the scaling law and the stagnation of hardware upgradation in terms of saving both computation and data. More favorably, our method is seamlessly compatible with other compression methods such as pruning, distillation, and low-rank decomposition to further amplify the compression efficacy. Systematic and comprehensive experiments confirm the efficacy of Hyper-Compression. In summary, our contributions are threefold:

i) We propose to use the trajectory density in the mathematical ergodic theory to realize a neural network compression, referred to as Hyper-Compression. This idea is not a modification of the existing approaches but a fundamental innovation in model compression research. To the best of our knowledge, this is the first time that ergodic theory is used in model compression, which creates many new research and translation possibilities.
    
ii)  In this new framework, we propose to cast the irrational winding as the suitable ergodic transformation, derive its error bound theoretically, and design engineering twists to make this new kind of model compression algorithm pragmatic.
    
iii)  Systematic and comprehensive experiments on large models such as the LLaMA \cite{touvron2023llama} series, {\color{black}Qwen \cite{ahmed2025qwen} series} and small models such as {\color{black}ResNet \cite{he2016deep} series}, UNet \cite{ronneberger2015unet} and Mobile-Net \cite{howard2019searching} confirm the competitiveness of the proposed Hyper-Compression, which is friendly to agile deployment in both academia and industry.


\section{Related Work}

\subsection{Model Compression}
Model compression \cite{zhu2023survey} plays a crucial role in the widespread deployment of deep learning models in a myriad of different settings. There are mainly four different classes of techniques including \textit{pruning}, \textit{quantization}, \textit{low-rank decomposition}, and \textit{knowledge distillation}. As Table \ref{tab:principle} shows, we highlight that Hyper-Compression is a fundamentally different method that performs the low-dimensional transformation for the target parameters. Such an essential difference can open a lot of doors for research and translation opportunities. We summarize model compression techniques below, accompanied by the recent advances in compressing large language models. Due to the limit of pages, we can only cover a few representative works. Moreover, because our proposed Hyper-Compression here focuses on post-training compression, and knowledge distillation is to a large extent based on training, we do not review articles on knowledge distillation here.
\begin{table}[htbp]
\vspace{-0.2cm}
  \centering
  \caption{Hyper-Compression compresses a model based on an essentially different mechanism, which performs the low-dimensional transformation for the target data.}
  \scalebox{1.1}{
  \begin{tabular}{lc}
\Xhline{3\arrayrulewidth}
      Method & Core Principle   \\
     \hline
       Pruning &  Sparsity  \\
  Quantization &  Low-precision \\
    Low-rank Decomposition  & Low-rank \\
    Knowledge Distillation & Knowledge Transfer \\
    \textcolor{cyan}{Hyper-Compression} & \textcolor{cyan}{Low-dimensionality} \\
\Xhline{3\arrayrulewidth}
    \end{tabular}}%
    \label{tab:principle}
\end{table}

\textbf{Pruning} is a method to achieve sparsity in neural networks, which involves removing unimportant synapses, neurons, layers, and even blocks from a neural network. Therefore, a good amount of the pruning research is dedicated to design different evaluation metrics to find which part of a network is unimportant. Slight pruning can also lead to better generalization in addition to compression, while heavy pruning often needs meticulous retraining to avoid high performance loss. Pruning is divided into the unstructured and structured. Unstructured pruning usually results in an irregular model composition. To truly harvest gains in inference time and parameter saving, users need specialized storage and computation schemes.

Unstructured pruning targets individual parameters, without the need of considering the internal structures of a model. SparseGPT  \cite{frantar2023sparsegpt} turned the pruning problem into a set of extremely large-scale instances of sparse regression, thereby avoiding the inversion of each matrix. Thus, SparseGPT could compress models of 10-100 billion parameters in just a few hours. Wanda \cite{sun2023simple} simultaneously considers weights and activations in pruning. This technique was motivated by an observation in \cite{dettmers2208llm2022} that a small subset of hidden features are substantially large in magnitude. Therefore, they augmented the standard weight magnitude with the input activations to evaluate the importance of weight. 

Structured pruning usually removes the entire neurons, filters, and blocks of a network, thereby leading to realistic compression and acceleration. Unlinke the conventional methods, EBert \cite{liu2021ebert} incorporated a predictor to dynamically pinpoint and remove unimportant heads in multi-head self-attention layers and unimportant structured computations in fully-connected networks, respectively, when inferring each batch of samples. LLM-shearing \cite{xia2023sheared} found that the pruned model has an imbalanced performance across different domains and tasks, and proposed to dynamically load data batch from each domain in proportion to its rate of performance reduction in that domain. K-Prune \cite{park2023accurate}, a retraining-free structured pruning method, used the knowledge loss to measure the importance of heads in transformers, where the knowledge loss is defined as the faithfulness to the soft labels of the original model. FLAP \cite{an2024fluctuation} found that certain channels of hidden state features exhibit structured sample stability, and designed a structured pruning metric to identify whether the output feature map is easy to recover when a column of weight matrix is eliminated. Then, FLAP adds an additional bias term to recover the output feature.

\textbf{Quantization} reduces the precision of weights and activations in a neural network by turning the high bit-width numbers into lower bit-width ones. Quantization can directly diminish memory usage and accelerate computations by using fixed-point arithmetic. It enables faster inference on hardware with limited computational resources. Quantization entails quantization-aware training (QAT  \cite{liu2023llm}) and post-training quantization (PTQ \cite{huang2024billm}). The former emphasizes the combination of quantization and training.  
The latter quantizes parameters when the network training is completed, which does not require modifications to the network structure but may induce the extra cost of retraining. Therefore, here we mainly discuss PTQ, which is divided into weight quantization and weight\&activation quantization.

LLM.int8() \cite{dettmers2208llm2022} initially employs vector-wise quantization alongside distinct normalization constants for every inner product within the matrix multiplication to quantize most of the features. Moreover, LLM.int8() denotes outlier feature dimensions by a 16-bit matrix multiplication, while over 99.9\% of values are still multiplied in 8-bit. OPQ \cite{frantar2022optimal} is the layer-wise compression method, which performs quantization layer-by-layer and minimizes the error between the pre-activation by the original full-precision and the quantized weight matrix in each layer. SmoothQuant \cite{xiao2023smoothquant} is a training-free, accuracy-preserving, and general-purpose post-training quantization solution. It smoothens the activation outliers by a mathematical scaling transformation before the normal quantization:
\begin{equation}
    Y= (X\texttt{diag}(s)^{-1})\cdot (\texttt{diag}(s)^{-1}W),
\end{equation}
where $X$ is the activation and $W$ is the weight matrix. AWQ \cite{lin2024awq} highlights that not all weights are equally crucial for large language model (LLM) performance. Only 0.1\%-1\% of weights are significant; bypassing quantization for these salient weights can greatly reduce quantization loss. To identify these salient channels, they examined activation distributions rather than weight distributions. To avoid inefficient mixed-precision implementations, they analyzed weight quantization errors and determined that scaling up salient channels can decrease their relative quantization error. Atom \cite{zhao2024atom} also designs a customized CUDA kernel that utilizes low-bit tensor cores, besides the common quantization operations such as mixed precision and grouped quantization. Atom improves end-to-end throughput (token/s) by up to $7.73$ times compared to the FP16.

\textbf{Low-rank decomposition} involves approximating weight matrices or tensors in a neural network by decomposing them into low-rank ones. This technique reduces the number of parameters in the model, leading to a more compact representation.

Low-Rank Adaptation (LoRA, \cite{hulora}) posits that the adjustments in weights during model adaptation exhibit a low "intrinsic rank." This approach enables the training of certain dense layers in a neural network indirectly by optimizing the gradients into a low-rank decomposition matrix in those layers during adaptation, while keeping the pre-trained weights unchanged. TensorGPT \cite{xu2023tensorgpt} introduced the approach of tensorizing and decomposing each token embedding, rather than treating the entire embedding matrix as a single entity. It was the first method to apply the tensor train decomposition for model compression, achieving a reduction in the number of parameters by a factor of 2.31.

\subsection{Implicit Neural Representation}

Implicit Neural Representation (INR) \cite{park2019deepsdf} has emerged as a transformative approach in computer graphics, computer vision, and machine learning. This technique leverages neural networks to represent complex shapes and images without explicitly storing the geometry or pixel data. Early work in this area focused on using neural networks to learn mappings from coordinates to values. For instance, the pioneering work by \cite{park2019deepsdf} introduced the concept of using multi-layer perceptrons (MLPs) to represent 3D shapes as continuous functions, allowing for high-resolution representations without traditional mesh-based methods. Recently, the implicit neural representation is increasingly used for data compression, such as COIN \cite{dupont2021coin} and NERV++ \cite{ghorbel2024nerv++}.

Our Hyper-Compression can also be regarded as a kind of implicit representation, but Hyper-Compression is not based on a neural network. Moreover, to the best of our knowledge, no INR work is directly used for model compression. We think this is because when a model is large, learning a huge amount of parameters shall encounter problems such as slow convergence and considerable performance drop.

\subsection{Hypernet}

The concept of hypernetworks was first articulated by \cite{ha2017hypernetworks}, who proposed a framework where a neural network (the hypernetwork) generates the weights of another network. This meta-learning approach allows for dynamic adaptation, enabling a single hypernetwork to cater to multiple tasks by producing tailored weights. The hypernetwork architecture is particularly advantageous for scenarios where training multiple models is computationally expensive or infeasible. Hypernet as a meta-learning method has been successfully applied in a plethora of fields including few-shot learning \cite{rusumeta} and domain adaptation \cite{volk2022example}, to name a few.

Our Hyper-Compression is essentially different from hypernetworks in three dimensions. First, as mentioned earlier, Hyper-Compression is not based on a network. Instead, it generalizes the idea of hypernetwork to the hyperfunction. Second, hypernetworks highlight the control of the target network, while the Hyper-Compression is an inverse process that compresses weights of the target networks into fewer parameters of the hyperfunction. Third, Hyper-Compression is specific to model compression. In contrast, to the best of our knowledge, no hypernet is directly applied for model compression for reasons similar to INRs.


\section{Hyper-Compression and Ergodic Theory}

\subsection{Hyper-Compression} 
A small human genome can remarkably encode the development of a human brain to the scale of billions of neurons and trillions of synaptic connections \cite{stanley2009hypercube,shuvaev2024encoding}. This observation suggests the existence of an implicit mapping from genotype to phenotype, capable of expanding a limited number of genetic instructions into a vast array of biological substances. Inspired by this efficient genetic encoding, the hypernet \cite{chauhan2024brief} uses a small network (genotype, called hypernet) to control the design of the target network (phenotype) including architectures and weight distributions. 

Furthermore, we generalize the idea of hypernet to a parameterized function (hyperfunction). The biological observation is the existence of a mapping, regardless of whether it is denoted by a network or other forms of functions. We consider the hyperfunction in the setting of model compression. With a hyperfunction that can represent a large network with a few parameters, we ponder using the hyperfunction to do the model compression, \textit{i.e.}, one stores a hyperfunction about the model and queries its weights when needed. Mathematically, we use a hyperfunction $w_n=h(\theta;n)$ to encode the relationship between locations and weights of the target network, where $w_n$ is the $n$-th parameter of the target network, and $\theta$ is the hyperparameters of $h$. Instead of directly compressing weights, this novel perspective converts the model compression problem into the problems of finding their low-dimensional representation ($\theta$ has few elements than $\{w_n\}_{n=1}^N$). We refer to model compression using hyperfunction as the Hyper-Compression. Under the umbrella of Hyper-Compression, \textit{the important question is to design a suitable hyperfunction?}

\subsection{Ergodic Theory} 

We underscore that designing a suitable hyperfunction is an open-ended question. Generically, the selection of the hyperfunction should balance the compression time, computational complexity, and the compression rate. Unexpectedly, we find the connection between the hyperfunction and ergodic theory to address this question from a unique angle. A branch of ergodic theory \cite{cornfeld2012ergodic} studies on what conditions \textbf{a low-dimensional dynamic system's trajectory can cover all points in a high-dimensional space}. If a low-dimensional system's trajectory can cover all points in a high-dimensional space, from an engineering perspective, it is feasible to use low-dimensional curves to approximate high-dimensional points. Thus, we can use fewer parameters to denote more parameters via explicitly characterizing the target point with the low-dimensional dynamic system's trajectory. Now, we chart the pathway formally: 

\begin{definition}[\cite{cornfeld2012ergodic}]
A measure space is a triplet $(X, \mathscr{B}, \mu)$, where
1. $ X$  is a set. 2. $\mathscr{B}$ is a $\sigma$-algebra: a collection of subsets of $X$ which contains the empty set, and which is closed under complements, countable unions and countable intersections. The elements of $\mathscr{B}$ are called measurable sets.
3. $ \mu: \mathscr{B} \rightarrow[0, \infty]$, called the measure, is a $\sigma$-additive function: if $ E_{1}, E_{2}, \ldots \in \mathscr{B} $ are pairwise disjoint, then  $\mu\left(\bigcup_{i} E_{i}\right)=\sum_{i} \mu\left(E_{i}\right) $. If $\mu(X)=1$, then we say that $\mu $ is a probability measure and $ (X, \mathscr{B}, \mu) $ is a probability space.
\end{definition}

\begin{definition}[\cite{cornfeld2012ergodic}]
A measure-preserving transformation is a quartet $(X, \mathscr{B}, \mu, T) $, where $ (X, \mathscr{B}, \mu)$ is a measure space, and
1.  $T$  is measurable:  $E \in \mathscr{B} \Rightarrow T^{-1} E \in \mathscr{B} $;
2.  $m$  is  $T$ -invariant:  $\mu\left(T^{-1} E\right)=\mu(E)$  for all  $E \in \mathscr{B} $.
\label{def:mpt}
\end{definition}

\begin{definition}
 $T$  is said to be ergodic if and only if:
for all  $A \in \mathcal{B} $: $T^{-1}(A)=A \Longrightarrow \mu(A) \in\{0,1\}$.
\label{def:ergodicity}
\end{definition}

\begin{theorem}
If $T$ is ergodic, then 
for all  $A \in \mathcal{B}$:
$\mu(A)>0 \Longrightarrow \mu\left(\bigcup_{n=1}^{\infty} T^{-n}(A)\right)=1$.
\label{theorem_main}
\end{theorem}


\begin{proof}
Theorem \ref{theorem_main} is essentially the equivalent definition of the ergodicity. Please refer to the proofwiki \footnotemark[3] for detailed proof.
\end{proof}

\begin{figure}[h]
\centering
\includegraphics[width=0.6\linewidth]{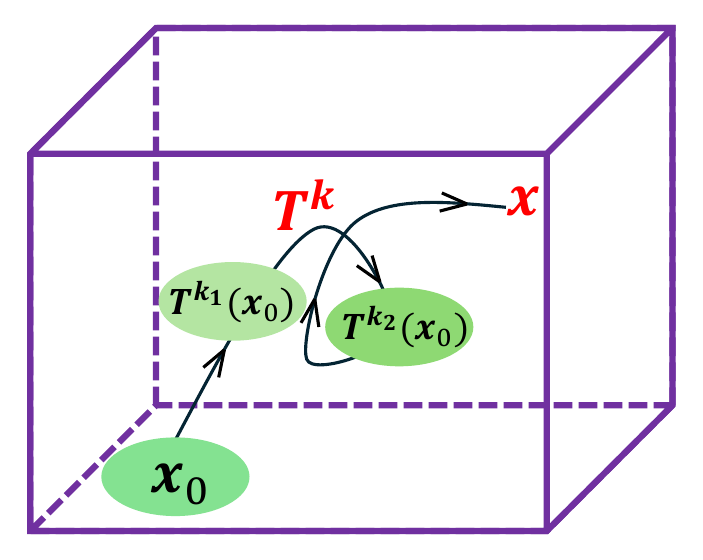} 
\caption{\footnotesize Based on ergodicity, we can construct a deterministic relationship between the one-dimensional trajectory quantity $k$ and the high-dimensional target point $\x \in \mathbb{R}^N$. Thus, we can compress $\x$ into $k$, and decompress $\x$ from $k$ based on $T^k(\x_0)$.} 
\label{ergodic_theory_based}
\vspace{-0.5cm}
\end{figure}



\begin{corollary}
Suppose the $\texttt{supp}(\mu)$ is of full measure, then given a set $A$ with $\mu(A)>0$, for any target point $\x \in \mathbb{R}^N$ and $\epsilon>0$, there exists a point $\x_0 \in A$ and $k \in \mathbb{Z}$ such that 
    \begin{equation}
        \Vert T^k(\x_0)-\x \Vert < \epsilon.
    \end{equation}
Without loss of generality, we define the norm $\Vert \b \Vert=\Vert \b \Vert_2=\sqrt{\sum_i b_i^2}$.    
\label{cor:approximation}    
\end{corollary}

\begin{proof}
Let two sets $B_{\epsilon/2}(\x_0) \in A$ and $B_{\epsilon}(\x) \in A$, where $B_\epsilon(\y)=\{\z~|~\Vert \z-\y\Vert<\epsilon\}$. Due to Theorem \ref{theorem_main} and the fact that $T$ is measure-preserving, there exists $k$ such that $T^k(B_{\epsilon/2}(\x_0)) \in B_{\epsilon}(\x)$ almost for every point from $B_{\epsilon/2}(\x_0)$, which concludes the proof.
\end{proof}


Let us interpret Theorem \ref{theorem_main} and Corollary \ref{cor:approximation}, with an emphasis on how it is related to the model compression problem. First, Definitions \ref{def:mpt}-\ref{def:ergodicity} and Theorem \ref{theorem_main} are purely set-theoretic, and hold true for arbitrarily high-dimensional space. This means that Corollary \ref{cor:approximation} also holds true for an arbitrarily large $N$. Second, Theorem \ref{theorem_main} shows that when the measure-preserving transformation is ergodic, even though $A$ is a tiny set, as long as its measure is positive, applying $T$ on $A$ recursively will create a one-dimensional trajectory which can fill the entire space $X$. As a result, per Corollary \ref{cor:approximation}, we can construct a deterministic relationship between the one-dimensional composition number $k$ and the high-dimensional target point $\x \in \mathbb{R}^N$, as shown in Figure \ref{ergodic_theory_based}. Third, based on Corollary \ref{cor:approximation}, we can compress a group of parameters $\x \in \mathbb{R}^N$ into a number $k$ by 
\begin{equation}
    \x \approx T^k(\x_0),
\end{equation}
where $\x$ is the target. Element-wise, we have 
\begin{equation}
    x_n \approx [T^k(\x_0)]_n, n = 1, 2, \cdots, N.
\end{equation}
Again, the prerequisite of this kind of compression is $\texttt{memory}(k)<\texttt{memory}(\x)$.

Theorem \ref{theorem_main} and Corollary \ref{cor:approximation} provide a natural and elegant way to encode a group of numbers into one number. However, we should only be cautiously optimistic about its effect in compression, since simply stitching digits of numbers can also ensemble a group of numbers into one number $(p_1p_2p_3,q_1q_2q_3)\to p_1p_2p_3q_1q_2q_3$. What matters is whether we can find a short $k$ to approximate $x_1, \cdots, x_N$ based on Eq. \eqref{eqn:core}, such that the memory footprint of the former is smaller than the latter. 

\textbf{Remark 1}. The idea of compressing parameters via ergodic theory can also be generalized to continuous dynamic systems which, for example, define the transformation $T$ with differential equations. In this case, the composition number $k$ turns into the trajectory length. In addition, $T$ can be either parametric or non-parametric. Since $T$ is universal for all the given $\x$, parameters defining $T$ only have a moderate impact on the compression rate but may provide more flexibility in compressing parameters.

\vspace{-0.1cm}
\subsection{Specific Case}

Earlier, we outlined a generic framework for the Hyper-Compression via ergodic theory. The next step is to examine the specific methodologies employed within this framework. Besides the basic ergodicity, the most important question is \textit{which dynamic system fits this specific compression issue most?} By addressing this question, we can translate the theory to compress the network parameter vector $\w$.

We think that the following famous theorem casts a good candidate, and we will explain later.

\begin{theorem}[\cite{katok1995introduction}]
\label{theorem1}
Suppose $a_1,\cdots,a_N$ are irrationally independent, 
for any given set $\{w_n\}_{n=1}^N \subseteq [0,1]$ and $\epsilon > 0$, there exists a value $\theta^* \in [0,+\infty)$ such that
\begin{equation}
    |w_n-\tau(\theta^*a_n)|<\epsilon,\; n=1,\cdots,N,
\label{eqn:core}    
\end{equation} where $\tau(z)=z-\lfloor z \rfloor$.
\end{theorem}

\begin{figure}[htb]
\begin{center}
\includegraphics[width=\linewidth]{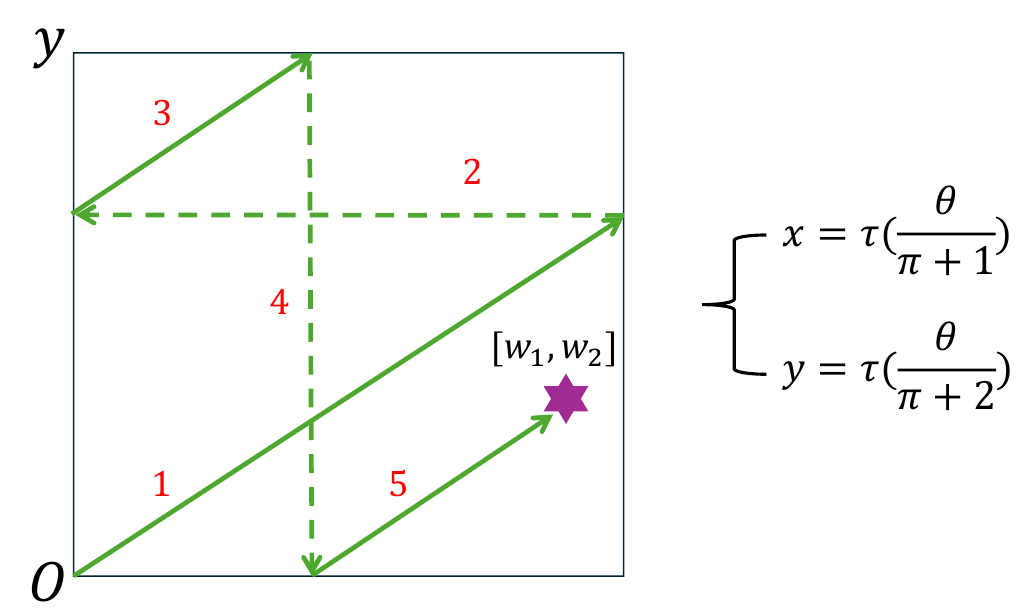}
\caption{\footnotesize A two-dimensional example to explain Theorem \ref{theorem1} that one-dimensional continuous dynamic system can fill the high-dimensional space. The irrational independence ensures that the trajectory never overlap. Thus, given a two-dimensional point $[\w_1,\w_2]$, we can find the corresponding $\theta$ to approximate $[\w_1,\w_2]$ by $[\tau(\theta/(\pi+1)), \tau(\theta/(\pi+2))]$. For example, $[0.07405, 0.00623] \approx [\tau(108/(\pi+1)), \tau(108/(\pi+2))]$. }
\label{explanation}
\end{center}
\end{figure}


\vspace{-0.3cm}
It can be seen that Theorem \ref{theorem1} is a special case of Eq. \eqref{cor:approximation}: Here, the initial point $\x_0$ is the original point. In addition, because Theorem \ref{theorem1} is a continuous dynamic system, the trajectory parameter $\theta$ corresponds to the composition number $k$ of Eq. \eqref{cor:approximation}. Theorem \ref{theorem1} is a corollary of Lemma 19 from \cite{zhang2022deep}. As shown in Figure \ref{explanation}, the irrational independence ensures that the trajectory never overlaps so that it can eventually fill the high-dimensional space, when $\theta$ goes to the infinity. In this theorem, the irrational independence condition is readily fulfilled; for instance, one can simply select $a_k = 1/(\pi + k)$. Later, we empirically and theoretically illustrate that adjusting $a_k$ is instrumental in reducing the approximation error and minimizing the performance drop.

The main reason why we use Theorem \ref{theorem1} is its simplicity, which makes it convenient to solve the trajectory parameter given the target and fast to restore the parameters in decompression. In doing so, we can both compress and restore a model fast, which is favorable for compressing large-scale models and time-sensitive scenarios, respectively.

The entire algorithmic development is anchored in Eq. \eqref{eqn:core}. Mathematically,
given a weight vector $\w$, we define the Hyper-Compression via Theorem \ref{theorem1} as 
\begin{equation}
    \mathcal{H}(\w): = \arg \min _{\theta \leq \Theta, ~ \theta, \Theta \in \mathbb{Z}} \Vert\w-\tau(\Delta\cdot\theta\cdot \a)\Vert
    \label{eqn:compress}
\end{equation}
and
\begin{equation}
    \mathcal{H}^{-1}(\theta): = \tau(\Delta\cdot\theta\cdot \a).
    \label{eqn:restore}
\end{equation}
Here, i) we slightly abuse the symbol $\theta$: $\theta$ is an integer here. We discretize $\theta$ with a fixed step size that is a natural number to make $\theta$ an integer; ii) We also slightly abuse the symbol $(\cdot)^{-1}$. Technically, $\mathcal{H}^{-1}$ is not the reverse function of $\mathcal{H}$. It just denotes the inverse process of $\mathcal{H}$. We use $\mathcal{H}^{-1}$ for the simplicity of notation.

\textbf{Remark 2}. We underscore that though we select Theorem \ref{theorem1} as a specific ergodic transformation in the framework of Hyper-Compression, what is the best ergodic transformation is indeed an open question. We intend to use a specific ergodic transformation to evaluate the feasibility of the overall idea. We think that selecting ergodic transformation should fully take into account the distribution of a model's parameters to maximize the compression rate. For example, one can adjust $T$ to force the trajectory to visit more frequently regions where most parameters are populating. 

In addition, since our method is purely centered on approximating a number vector, it is not the opponent of other classes of model compression algorithms like pruning, decomposition, and distillation. It can organically synergize with them such as stacking a Hyper-Compression on top of them to further escalate the compression efficacy. 

Lastly, Theorem \ref{theorem1} needs that $a_1,a_2,\cdots,a_N$ are irrational numbers. However, in computers that are restricted by the precision limit, all numbers are rational. Theoretically, this will cause the trajectory defined in Theorem \ref{theorem1} to fail covering the entire space. But we actually use Theorem \ref{theorem1} to do approximation, machine precision has been much higher than we need to provide a satisfactory approximation for network parameters.

\section{Error Analysis and Algorithmic Design}
\label{algorithmic}

In this section, we derive the error bound for using Theorem \ref{theorem_main} to compress a single-layer network, which strengthens our understanding for the characters of this kind of compression. Notably, our result can be extended to the multi-layer networks directly. Therefore, we do not put the error bound for these networks for conciseness. Our theory is largely based on \cite{zhang2023post}. Hence, we inherit notations from \cite{zhang2023post} for simplicity. Next, we attempt to materialize the idea of using Theorem \ref{theorem_main} to do model compression. Drawing insights from our theoretical analysis,
we propose four engineering twists to prototype a pragmatic model compression algorithm. 

\subsection{Error Analysis}

In the single-layer network, we define the input data $X \in \mathbb{R}^{m \times N_0}$, the weight matrix $\W \in \mathbb{R}^{N_{0} \times N_{1}}$, and the activation function $\phi: \mathbb{R} \to \mathbb{R}$, which means there are $m$ samples, each with $N_0$ dimensions, and $N_1$ neurons in this single-layer network. Following the divide-and-conquer strategy, estimating the error bound of a single neuron can naturally lead to a good estimation for the entire single network. Without loss of generality, let $\w=[w_1,w_2,\cdots,w_{N_0}]$ be the weight vector of a neuron. With Eqs. \eqref{eqn:compress} and \eqref{eqn:restore}, the compression and decompression can be expressed as
\begin{equation}
    \left\{\begin{array}{l}
\theta=\mathcal{H}\left(\w\right) \\
\q = \mathcal{H}^{-1}(\theta), \\
\end{array}\right.
\label{compre_decompre}
\end{equation}
$\q$ is the decompressed weight vector. 

Since the Hyper-Compression is not lossless, there must be a discrepancy between $\w$ and $\q$, which results in an error in the neuron's final output. To characterize this error, we first characterize the pre-activation error of a neuron. Specifically, we denote pre-activation error recursively in the form of error accumulation.

\begin{equation}
    \left\{\begin{array}{l}
\u_0 = \mathbf{0} \\
\u_t=\u_{t-1}+w_t \X_t-q_t \X_t \\
\u_{N_0}=\X \w-\X \q, 
\end{array}\right.
\end{equation}
where $\u_t$ is the accumulated error generated from $1^{st}$ to the $t$-th weight in the neuron. Thus, $\u_{N_0}=\X \w-\X \q$ is the total error. Our goal is to describe how $\u_{N_0}$ is bounded. The error is measured by the $L_1$-norm, \textit{i.e.}, $\left\|\u\right\|_{1}=\left\|\u\right\|=\sum_{i} |u_i|$.

\begin{theorem}
Given the input data $X \in [0,1]^{m \times N_{0}}$, suppose that each element of the column $X_{t}$ of $X \in \mathbb{R}^{N_{0}}$ is i.i.d. drawn from a uniform distribution over $[0,1]$, and $\w \in [0,1]^{N_{0}}$, we compress and decompress $\w$ based on Eq. \eqref{compre_decompre} with the step size $\Delta>0$ and the maximum integer $\Theta$. Then, there exists $c\in(1/2,1]$, for any $\epsilon<1$, we have $\q$
\begin{equation}
\mathrm{P}\left(\|\X \w-\X \q\| \leq 2cmN_0 \epsilon \right) 
\geq 1-e^{-cmN_0 \epsilon},
\label{eqn:pre_error_analysis}
\end{equation}
where $\q$ is the recovered weight vector.
Furthermore, if the activation function $\varphi: \mathbb{R} \rightarrow \mathbb{R} $ is  $\xi$-Lipschitz continuous, that is,  $\mid \varphi(x)- \varphi(y)|\leq \xi| x-y \mid $ for all  $x, y \in \mathbb{R}$, then we have
\begin{equation}
\mathrm{P}\left(\|\varphi(\X \w)-\varphi(\X \q)\| 
   \leq 2cmN_0 \epsilon\xi \right) 
   \geq  1-e^{-cmN_0 \epsilon}
\label{eqn:error_analysis}
\end{equation}
\label{theorem_4}
\end{theorem}
\vspace{-0.8cm}

\begin{proof} 
$\u_t$ is composed of $w_t-q_t$ and $\X$. We first estimate them, respectively, and then we provide the estimation for $\u_{N_0}$ by integrating them.

i) Suppose that $\w \in [0,1]^{N_{0}}$, we compress $\w$ based on Eq. \eqref{eqn:compress} with the step size $\Delta>0$ and the largest integer $\Theta$. Then, there must exist a sufficiently large $\Theta$ such that for $t=1,2, \ldots, N_{0}$,
\begin{equation}
|w_t-q_t| \leq \epsilon.
\end{equation}
This is because when $\Theta$ is large, there will be one element $\{\tau(\Delta\cdot\theta\cdot \a)\}_{\theta=1}^{\Theta}$ sufficiently close to $\w$ due to the density of trajectory.

ii) Let $ U$  denote the uniform distribution on $[0,1]$ with Suppose that the random vector  $\X \in \mathbb{R}^{m}$ with each element randomly drawn from $U$. Then as $m\to \infty$, we have  
\begin{equation}
    \mathbb{P}(\| X\|<cm) \to 1,
\end{equation}
where $c \in (1/2,1]$. The $l_{1}$-norm of the vector  $\X$ is 
$\|\X\|_{1}=\sum_{i=1}^{n} X_{i}$,
where each $X_{i}$  is i.i.d. uniform on $[0,1]$ .
Because each element $X_i$ is independent from each other, the expected value of $\|\X\|$ is 
\begin{equation}
\mathbb{E}\left[\|\X\|\right]=m \cdot \mathbb{E}\left[X_{i}\right]=m/2.
\end{equation}
Using the Law of Large Numbers, as $m$ increases, the sum 
$\sum_{i=1}^{m} X_{i}$ converges almost surely to its expected value $\frac{m}{2}$. Thus, almost surely, $\| X\|<cm$.

Let $\alpha>0$, by Markov's inequality, one can get 
\begin{equation}
\begin{aligned}
\mathrm{P}\left(\left\|\u_{N_0}\right\| \geq \alpha\right) 
&=\mathrm{P}\left(e^{\eta\left\|\u_{N_0}\right\|} \geq e^{\alpha}\right) \leq e^{-\alpha} \mathbb{E} e^{\left\|\u_{N_0}\right\|} \\
& = e^{-\alpha} \mathbb{E} e^{\left\|\sum_{t=1}^{N_0} (\u_t-\u_{t-1})\right\|} \\
& \leq e^{-\alpha} \sum_{t=1}^{N_0} \mathbb{E} e^{\left\| \u_t-\u_{t-1}\right\|}\\
& \leq e^{-\alpha} \cdot e^{cmN_0 \epsilon}
\end{aligned}
\end{equation}

Let $\alpha=2cN_0m \epsilon$, we have
\begin{equation}
   \mathrm{P}\left(\left\|\u_{N_0}\right\| \geq 2cmN_0 \epsilon \right) \leq  e^{-cmN_0 \epsilon},
\end{equation}
which proves Eq. \eqref{eqn:pre_error_analysis}.  
Then, considering the Lipschitz continuity of the activation function $\varphi$, we have $\|\varphi(\X \w)-\varphi(\X \q)\| \leq \xi \|\X \w-\X \q \|$. Integrating it with Eq. \eqref{eqn:pre_error_analysis} proves Eq. \eqref{eqn:error_analysis}.  

\end{proof}

Eq. \eqref{eqn:error_analysis} characterize that the sample size $m$, the dimensionality $N_0$ and $\Theta$ dominates the reconstruction error $\varphi(\X \w)-\varphi(\X \q)$. It is straightforward to understand that when the sample size $m$ and the dimensionality $N_0$ go higher, more error will be accumulated. In addition, increasing $\Theta$ can add more points into the set $\{\tau(\Delta\cdot\theta\cdot\a): \theta \leq \Theta\}$, which will naturally increase the likelihood of finding a point closer to the given $\w$. 


\textbf{Remark 3.} When no post-hoc retraining is performed, Hyper-Compression can approximately achieve the same level of error but a higher compression rate compared to quantization. Let us take the \texttt{Int2} quantization as an example to illustrate this point. As Figure \ref{ergodic_theory_when} shows, we sample five points (A, B, C, D, E) from the trajectory of the dynamic equation. Because of the irrational direction, the projection of all these five points along both $x$ and $y$ axes will not overlap. Projections of these five points will divide both $x$ and $y$ axes into six segments, \textit{e.g.}, $\Xi_{1x},\Xi_{2x},\cdots,\Xi_{6x}$ and $\Xi_{1y},\Xi_{2y},\cdots,\Xi_{6y}$. We reasonably assume that each segment is at the same level of $1/(2^{\texttt{Int2}}+1)$. If we adjust the irrational direction to make points relatively evenly distributed in the 2D space, the error in approximating most 2D points at positions $x$ is usually no more than one segment, and the error will never exceed a summation of two consecutive segments. The Hyper-Compression achieves the same level of error as \texttt{Int2}. But the Hyper-Compression has a double compression rate, since it turns two numbers into an integer from $\{1,2,3,4,5\}$.

\footnotetext[1]{For more details, please visit the Hugging Face website of LiteLlaMA: \url{https://huggingface.co/ahxt/LiteLlama-460M-1T}.}

\footnotetext[2]{For more details, please visit the Hugging Face website of INCITE-Base: \url{https://huggingface.co/togethercomputer/RedPajama-INCITE-Base-3B-v1}.}

\footnotetext[3]{For more details, please visit the following website: \url{https://proofwiki.org/wiki/Equivalence_of_Defintions_of_Ergodic_Measure-Preserving_Transformation}.}

\footnotetext[4]{For more details, please visit the Hugging Face website of LLaMA-3.2-1B-Instruct: \url{https://huggingface.co/meta-llama/Llama-3.2-1B-Instruct}.}

\footnotetext[5]{For more details, please visit the Hugging Face website of Qwen1.5-7B-Chat: \url{https://huggingface.co/Qwen/Qwen1.5-7B-Chat}.}

\footnotetext[6]{For more details, please visit the Hugging Face website of DeepSeek-R1-Distill-Qwen-14B and 32B: \url{https://huggingface.co/deepseek-ai}.}

\begin{figure}[htb]
\begin{center}
\includegraphics[width=0.9\linewidth]{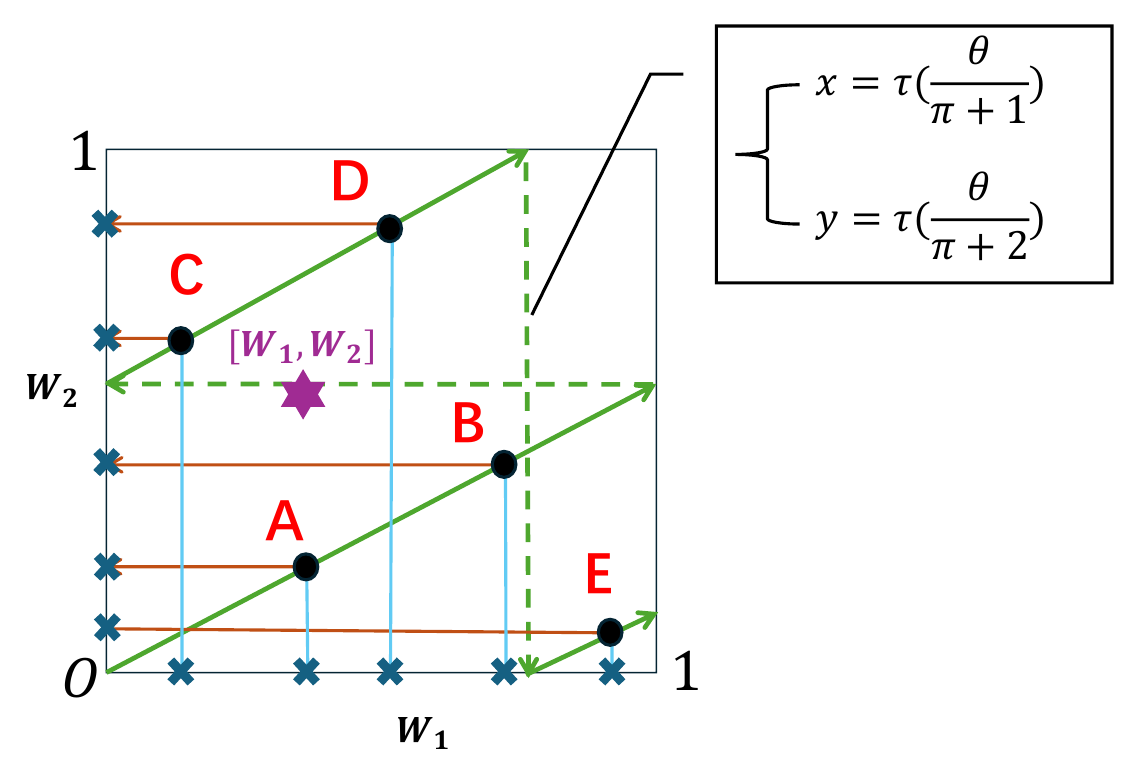}
\caption{\footnotesize When no post-hoc retraining is performed, Hyper-Compression can approximately achieve the same level of error but a higher compression rate compared to quantization. Five points (A, B, C, D, E) are equidistantly sampled from the one-dimensional trajectory. The projections of five points (A, B, C, D, E) along x and y-axes will not overlap due to the irrationality of $\pi$, which ensures the 2D space is fully partitioned to have a small approximation error.}
\label{ergodic_theory_when}
\end{center}
\vspace{-0.5cm}
\end{figure}

\subsection{Algorithmic Design}
For points in $[0,1]^N$, the core objective of the algorithm is, given the preset error $\epsilon$, to engineering as much as we can to rapidly achieve the smallest $\theta^*$ and the highest possible dimension $N$. First, because a large model usually has billions of parameters, searching $\theta^*$ for a group of $N$ parameters has to be fast to ensure that the hypercompression can scale. Second, an imprecise $\theta^*$ could lead to significant performance degradation in the restored networks, rendering the retraining. Therefore, the algorithm must ensure $\theta^*$ to be precise for the highest possible dimension $N$. Third, when time and precision allow, the algorithm should return the smallest $\theta^*$ for the highest compression ratio. We design a general-purpose model compression method that enjoys the following benefits (\textbf{PNAS}): \textbf{P}referable compression ratio, \textbf{N}o post-hoc retraining, \textbf{A}ffordable inference time, and \textbf{S}hort compression time, as summarized in Table \ref{tab:merit_comparison}. As Figure \ref{compression} shows, the compression is as follows:

\begin{figure}[h]  
\centering
\includegraphics[width=\linewidth]{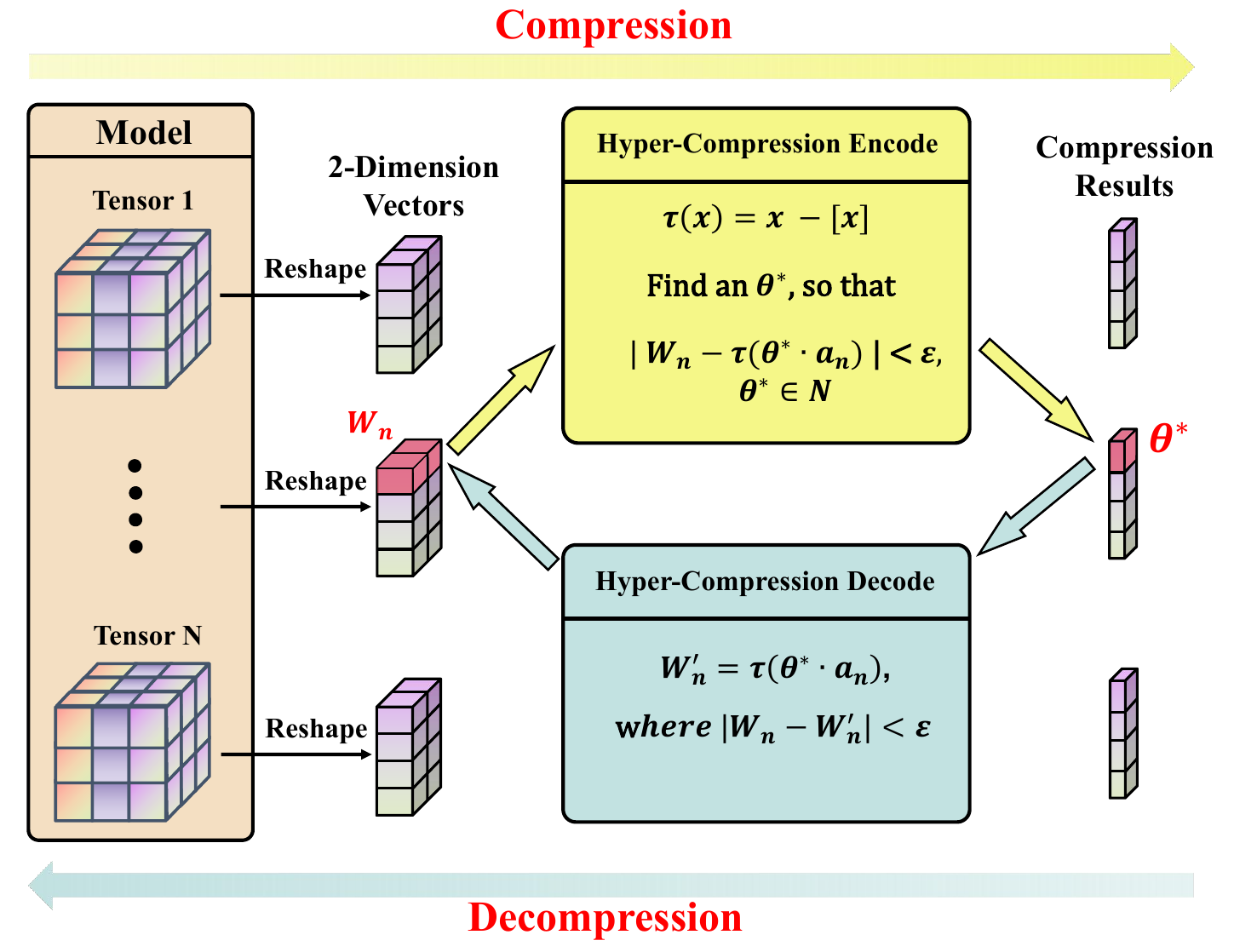} 
\caption{The overall flowchart of using Theorem \ref{theorem1} to compress parameters and restore parameters, which includes dividing parameters into groups and solving $\theta$ for each group.} 
\label{compression}
\end{figure}

First, we flatten all parameters into a one-dimensional vector $\w=[w_1,w_2,\cdots,w_N]$, where $N$ is the total number of weights, and then split it into groups $\left[\w^{(1)}, \w^{(2)}, \cdots, \w^{(G)}\right]$, where $\w^{(g)}=[w_{K(g-1)+1}, w_{K(g-1)+2},\cdots, w_{K(g-1)+K}]$ is of $\mathbb{R}^{1\times K}$, and $G$ is the number of groups. For any $\w^{(g)}$, there exists an integer $\theta^*_g$ such that  $|\w^{(g)}-\tau(\theta^*_g \cdot \a )|<\epsilon, g=1,\cdots,G$ and $\a = [a_1, a_2, \cdots, a_K]$. Finally, the vector $\left[\theta^*_1, \theta^*_2, \cdots ,\theta^*_G\right]$ represents a compression to $\w$. Take $K=2$ and $\w=[w_1,w_2,w_3,w_4]=[0.1,0.2,0.4,0.5]$ as an example, we split $\w$ as $[\w^{(1)}, \w^{(2)}]$, where $\w^{(1)}=[w_1,w_2]=[0.1,0.2]$ and $\w^{(2)}=[w_3,w_4]=[0.4,0.5]$, and derive $\theta^*_1$ and $\theta^*_2$ for $\w^{(1)}$ and $\w^{(2)}$, respectively. As Figure \ref{explanation} shows, we can find a $\theta$ such that $(\tau(\theta/(\pi+1)), \tau(\theta/(\pi+2)))$ can approximate a given point in the 2D space.

Given a parameter vector $\w$, $\Theta=\left[\theta^*_1, \theta^*_2, \cdots ,\theta^*_G\right]$ is a compression. However, we do not want $\max(\theta^*_i)$ to be a large number, as this would require more storage space. Therefore, we define an integer $U$ as the upper bound of $\theta^*_g$ to ensure the preferable compression efficacy. Next, $\left[\theta^*_1, \theta^*_2, \cdots, \theta^*_G\right]$ can be stored using the data type uint$m$, where $m = \lfloor\log_{2}(U)\rfloor + 1$, which is the minimum possible number of bits without causing data overflow. In other words, given a target two-dimensional point $\w^{(g)}$ and $U$, what we are doing is finding an integer $\theta^*_i \in \left[0, U\right]$ such that the error between $\tau(\theta^*_g \cdot \a )$ and $\w^{(g)}$ is minimized, where $\a=[a_1,a_2]$ is the irrational direction. As $U$ increases, potentially, more integers can be explored to make the error $|\w^{(g)}-\tau(\theta^*_g \cdot \a )|$ smaller. Each layer can select a different $U$, allowing important layers to choose a larger $U$ to ensure the approximation error remains sufficiently small.

\begin{table}[htbp]
  \centering
  \caption{The advantage comparison between Hyper-Compression and other model compression methods when compression ratio is high.}
\scalebox{1}{
  \begin{tabular}{lcccc}
  \Xhline{3\arrayrulewidth}
       & Data & Training & Compress. Time & Infer. Time \\
     \hline
       Pruning   &  \xmark &  \xmark &  \cmark & \cmark\\
  Quantization   &  \cmark &  \cmark &  \cmark & \cmark \\
    Decomposition    &  \xmark &  \xmark &  \xmark & \xmark\\
    Distillation  & \cmark &  \cmark & \xmark & \cmark\\
    Ours  & \xmark & \xmark & \cmark & \cmark\\
   \Xhline{3\arrayrulewidth}
    \end{tabular}}%
    \label{tab:merit_comparison}
\end{table}%

Now, we describe in detail our engineering twists in order to harvest the \textbf{PNAS} benefits.

\textbf{1. Translation, scaling, and adjusting irrational directions $\to$ \underline{P}referable compression ratio and \underline{N}o post-hoc fine-tuning}. We consider the distribution of weights in the original network by binning weights into different boxes based on magnitudes of weight values and adjusting the irrational directions to cover as many weights as possible.

\begin{itemize}

    \item In our algorithm, we set $K=2$ and replace the unit square $[0,1]^{2}$ with a more flexible box $[a, b] \times [c, d]$. Specifically, given $\w$, $[a, b] \times [c, d]$ is defined as a square with the side length $l$, centered at $(\bar{x_{i}}, \bar{y_{i}})$, where $(\bar{x_{i}}, \bar{y_{i}})$ is the centroid of two-dimensional points $\w^{(1)}, \w^{(2)}, \cdots, \w^{(G)}$. $l$ is a hyperparameter typically set to 0.01. Consequently, the function $\tau(x):=x-\lfloor x\rfloor $ in the ergodic theorem is generalized as $\tau([x,y]):= f([x,y]) + ([\bar{x_{i}}, \bar{y_{i}}] - [\frac{l}{2}, \frac{l}{2}])$, where $f(x) = x \bmod l$.

    \item If we construct a square with its center at the centroid $(\bar{x_{i}}, \bar{y_{i}})$ and the side length $l$, it is highly probable that many points in $\w^{(1)}, \w^{(2)}, \cdots, \w^{(G)}$ will fall outside this square. Therefore, we need to first scale these points into the square using a scaling factor $s$ during compression, and then scale the substituted points $\tau(\theta^*_g \cdot \a )$ back during decompression. While the error rate remains intact, this inversion will amplify the error $\epsilon$, where $\epsilon = |\w^{(g)}-\tau(\theta^*_g \cdot \a )|$. Assuming that $\w^{(F)}$ is the point farthest from the centroid  $(\bar{x_{i}}, \bar{y_{i}})$ in $\w^{(1)}, \w^{(2)}, \cdots, \w^{(G)}$, we define the farthest distance $l_f$ as $|\w^{(F)}-[\bar{x_{i}}, \bar{y_{i}}]|$, and then the scaling factor $s=\frac{l}{2\cdot l_f}$. The specific formula derivation is as follows:
    \begin{equation}
    \begin{aligned}
    \w^{(in)} & = C\w^{(in)} + C \\
        &  = C\w^{(out)} \cdot s + C \\
        &  = C\w^{(out)} \cdot \frac{l}{2\cdot l_f} + C,
    \end{aligned}
    \end{equation}
    where $C$ is the centroid node, $\w^{(out)}$ is a node ouside the squre, and $\w^{(in)}$ is the scaled node of $\w^{(out)}$. 
    
    Let $\w^{(in)*}=\tau(\theta^*_i \cdot \a )$. Thus, $\epsilon = \w^{(in)*} - \w^{(in)}$. The detailed scale-back process is shown as follows:
    \begin{equation}
    \begin{aligned}
        \w^{(out)*} & = \frac{C\w^{(in)*}}{s} + C \\
              & = \frac{\w^{(in)} + \epsilon}{s} + C.
    \end{aligned}
    \end{equation}

    Therefore, we can derive the error:
    \begin{equation}
    \begin{aligned}
        \texttt{error} & = |\w^{(out)} - \w^{(out)*}| = |C\w^{(out)} - C\w^{(out)*}| \\
             & = |C\w^{(out)} - \frac{C\w^{(in)*}}{s}| = |\frac{\w^{(in)*}- \w^{(in)}}{s}| \\
             & = |\frac{\epsilon}{s}| = |\frac{2 \cdot \epsilon \cdot l_f}{l}|
    \label{error}
    \end{aligned}
    \end{equation}
    
    From \eqref{error}, we can observe that the farthest distance $l_f$ affects the error of estimating $\w^{(out)}$ after scaling back.

    \begin{figure}[htb]
    \begin{center}
    \includegraphics[width=0.8\linewidth]{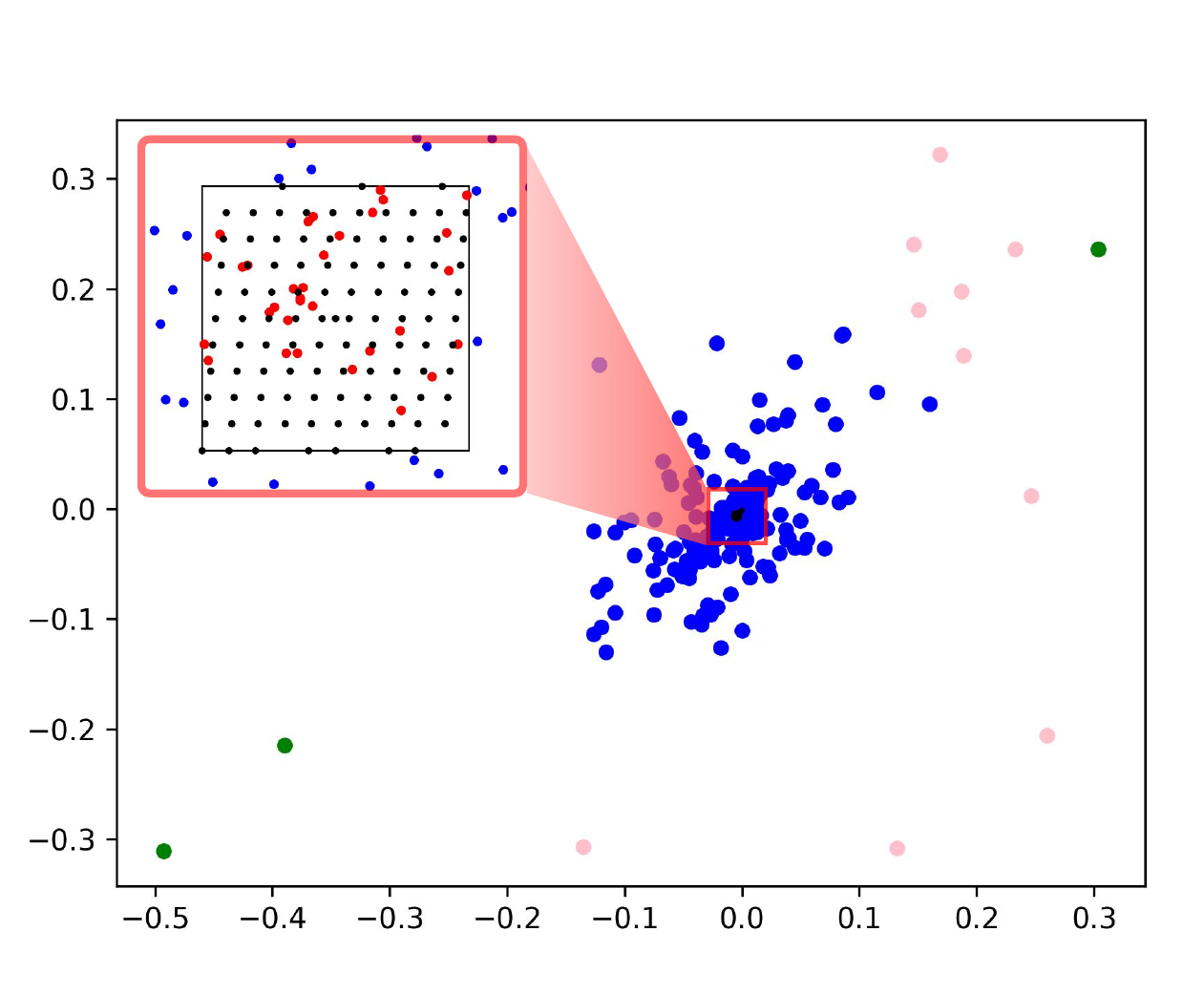}
    \caption{Given $K=2$, the points outside the center square are divided into three categories: blue points belong to the first category, pink points belong to the second category, and green points belong to the third category.}
    \label{layer_k}
    \end{center}
    \vspace{-0.5cm}
    \end{figure}

    
    Therefore, as shown in Figure \ref{layer_k} in our algorithm, we classify the points outside the square into $M$ categories based on their distances from the centroid $C$, where $M$ is a hyperparameter, so that the points in different categories can utilize different farthest distance $l_{fm}$ to define different scale factor $s_m$. This is better than using the global farthest distance $l_f$ to control the compression error. This method is highly effective in practice, as most points are close to the centroid.

    Specifically, assuming that the points outside the square are divided into $M$ categories and $\w^{(m)}$ is a point under the $m$-th category, the scaling factor $s_m$ for this point is defined as

    \begin{equation}
        s_m = \frac{l/2}{l/2 + (l_f / K) \cdot m }.
    \end{equation}

    The final compression $\theta_m^*$ is given by the following formula:

    \begin{equation}
    \theta_m^* = m \cdot U + \lambda_m^*,
    \end{equation}
    where $\lambda_m^*$ is the compression of the scaled $\w^{(m)}$ in the center square.

    \item $\a = [a_1, a_2] = \d \cdot I$ is a decomposition, where $\d$ and $I$ determine the direction and step size for the movement of curves defined, respectively. Since $\theta$ is an integer no more than $U$, at most we have $U+1$ points in two-dimensional space $(\tau(a_1\theta),\tau(a_2\theta)), \theta=0,1,\cdots,U$. In our experiment, we define an optimal vector $\a$ based on $U$, such that the distribution of $U+1$ points is more uniform, thereby minimizing the maximum error as much as possible. Specifically, 
    \begin{equation}
    \begin{aligned}
        \a & = \d \cdot I \\
        & = [\frac{l}{U},l]/\|\frac{l}{U},l\| \cdot I \\
        & = [\frac{l}{U},l]/\|\frac{l}{U},l\| \cdot \frac{l}{\sin{\alpha}\cdot \lfloor \sqrt{U} \rfloor},
    \label{vector_m1}
    \end{aligned}
    \end{equation}
    
    where $\tan{\alpha} = \frac{l}{l/U}=U$. As shown in Figure \ref{nodes_U}, it demonstrates that using \eqref{vector_m1} results in a more uniform distribution of the sample nodes compared to defining $\a$ as follows:
    
    \begin{equation}
    \a = [1/(\pi+1), 1/(\pi+2)].
    \label{vector_m2}
    \end{equation}
    
    \end{itemize}
    
\begin{figure}[hbt]
\begin{center}
\includegraphics[width=0.9\linewidth]{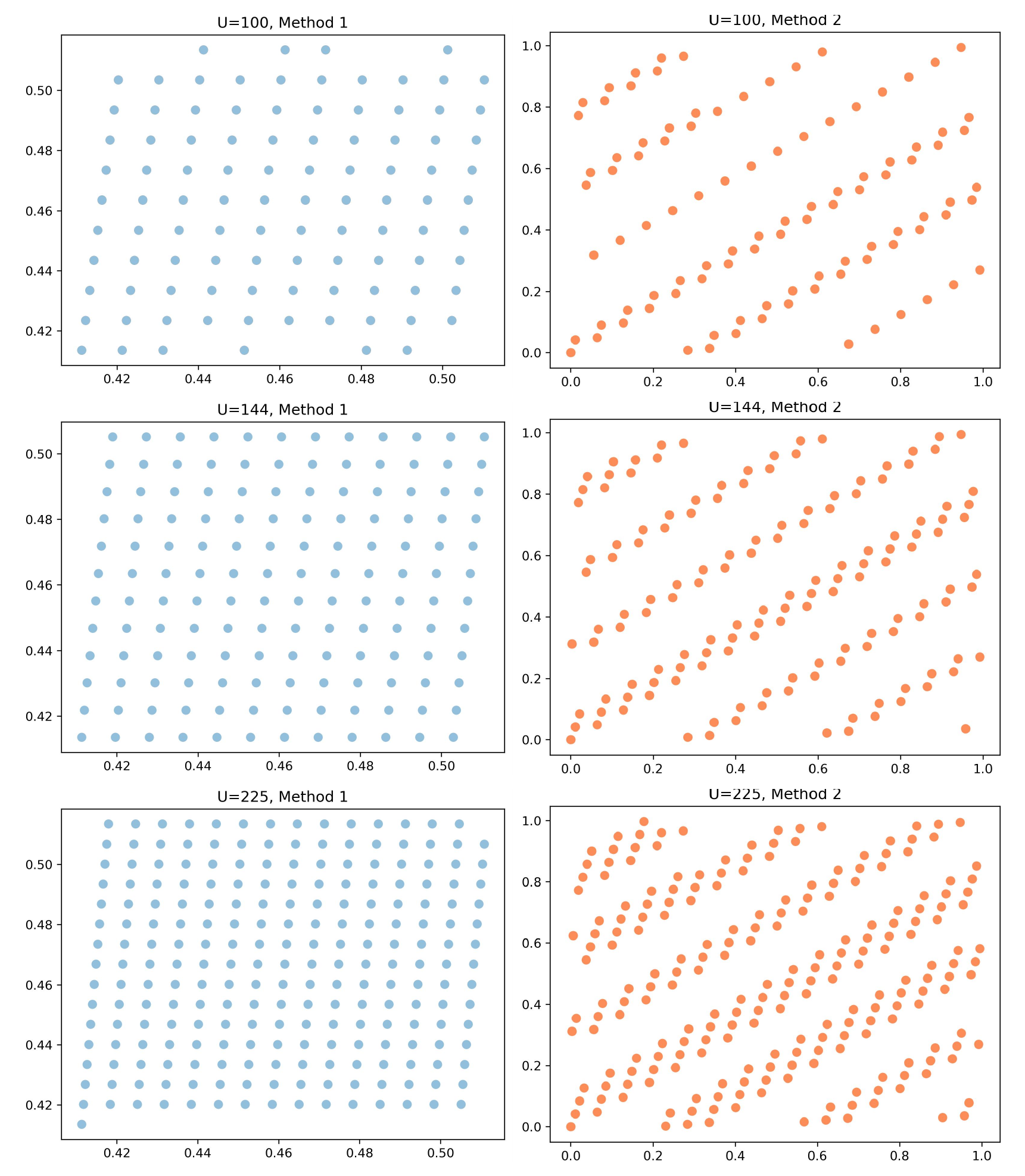}
\caption{Different $\a$ will result in different distribution of points. Upper three subfigures use \eqref{vector_m1}, while lower three figures utilize \eqref{vector_m2}.}
\label{nodes_U}
\end{center}
\end{figure}

\textbf{2. Simultaneous Inference and Decompression $\to$ \underline{A}ffordable inference time}. At first glance, one may think that Hyper-Compression suffers from a slow inference time, as this technique needs to restore parameters before inference, which adds another level of computation. Here, we leverages the intrinsic hierarchical structure of a network to greatly reduce the inference time. Our scheme parallelizes parameter decompression and inference. As shown in Figure \ref{Fig:trend}, while the parameters of later layers (except the first layer) are restored, the inference operation in earlier layers is also carried out simultaneously. As long as we can recover the parameters of the current layer before the inference arrives at this layer, there is no waste of time. Thus, theoretically, the inference time of using our algorithm only increases moderately. The increment is the time used to restore the parameters of the first layer.

\begin{figure}[htb]
\center{\includegraphics[width=\linewidth] {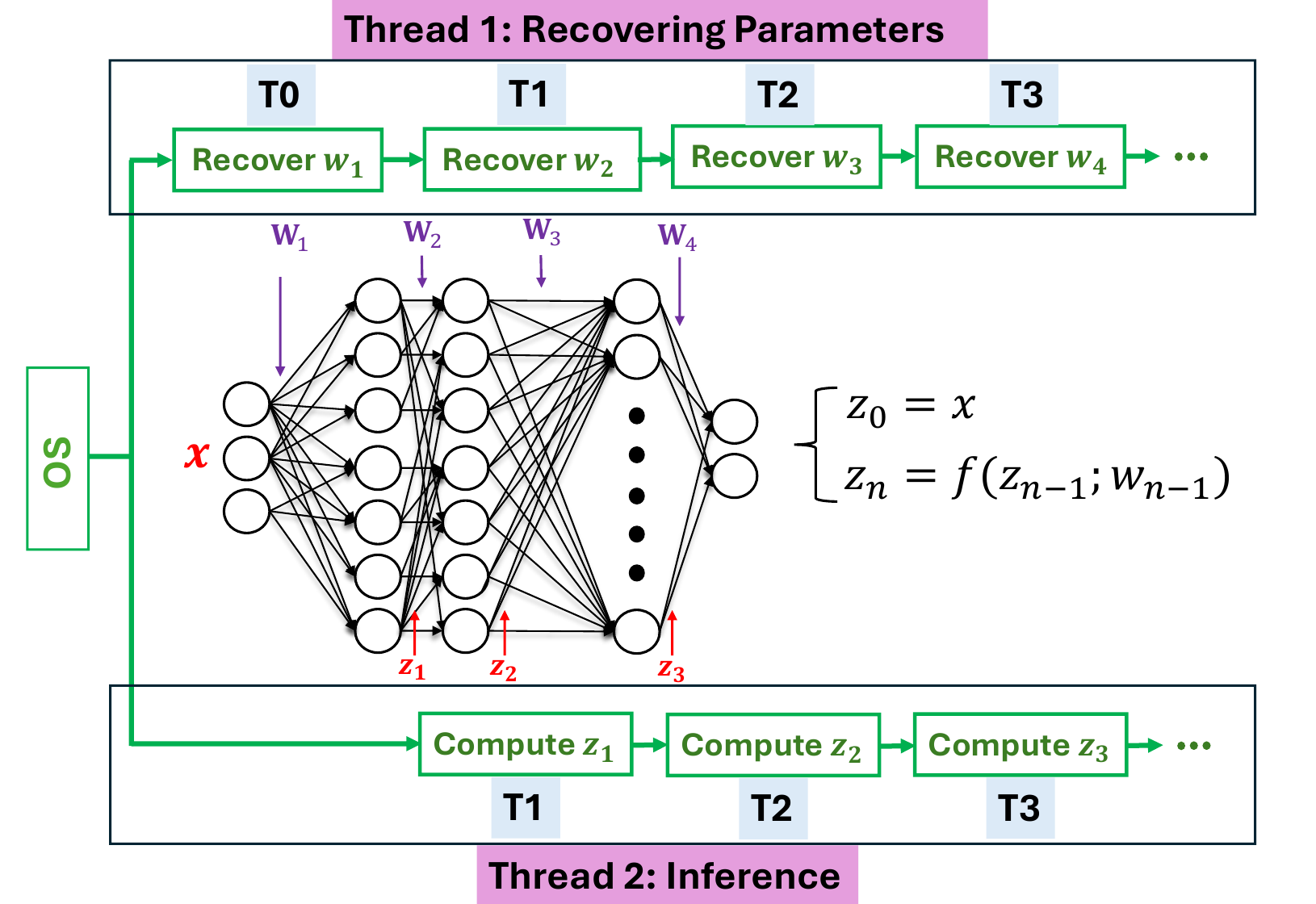}}
\caption{The rapid growth of LLM's size has outpaced the growth of GPU memory, creating challenges in serving these increasingly massive models.}
\label{Fig:trend}
\vspace{-0.5cm}
\end{figure}

Based on the above solution, we need to reduce the decompression time. First, the entire decompression process is implemented using matrix operations, which is significantly faster. Second, the process of reading files from storage and preprocessing them is relatively time-consuming. Therefore, we optimize this process, performing file reading and preprocessing operations only when the function is called for the first time, and store the processing results in the cache. When the function is called again, the preprocessed intermediate results are directly captured from the cache, which can further accelerate the inference.


\textbf{3. KD-Tree + Parallelization $\to$ \underline{S}hort compression time}.
\label{subsec:shortcomtim}
With the advent of large models, compression time becomes an important facet of evaluating a model compression algorithm. Here, we propose a suite of techniques to enable the proposed Hyper-Compression method to fast finish the compression for both CPU and GPU inference. 

\begin{itemize}
    \item When encoding the parameter vector $\w$, we first convert it into many points $\w^{(1)}, \w^{(2)}, \cdots, \w^{(G)}$. Then, given $\w^{(g)}$, we search $u^*$ from $\{\tau(1 \cdot \a ), \cdots,  \tau(U \cdot \a )\}$ to approximate $\w^{(g)}$. Next, instead of repeating searching for every point, we store $\tau(1 \cdot \a ), \cdots,  \tau(U \cdot \a )$ in the format of a KD-Tree \cite{ram2019revisiting} to turn the searching problem into the problem of finding the nearest neighbor. Thus, we can determine $\tau(u^* \cdot \a ) \in \{\tau(1 \cdot \a ), \cdots,  \tau(U \cdot \a )\}$ that is closest to $\w^{(g)}$.


    \item We extensively use matrix operations. Given a series of two-dimensional points $\w^{(1)}, \w^{(2)}, \cdots, \w^{(G)}$, we first calculate the scaling factor list $F$:

        \begin{equation}
        \begin{aligned}
            F &= [s_1, s_2, \cdots, s_G]^\top, \\
            s_g &= \frac{l/2}{l/2 + (l_f/M) \cdot m_g},
        \end{aligned}
        \end{equation}
        where $s_g \in F$ represents the scale factor of $\w^{(g)}$ and $m_g$ means the category of node $\w^{(g)}$. The points located within the center square are classified as category $m_g$.



        Let $\boldsymbol{m} = \left[m_1, m_2, \cdots, m_G\right]^\top$, $\w = \left[\w^{(1)}, \w^{(2)}, \cdots, \w^{(G)}\right]^\top$. Then, we can calcute the final compression $\theta^*$ as follows:


        {\color{black}
        \begin{equation}
        \begin{aligned}
        &C\w = \w^\top - [C, \cdots, C ]^\top, \\ 
        &O = C\w \cdot F + [C, \cdots, C]^\top \\ 
        &\u^* = \mathtt{KD} (O^\top, [\tau(1 \cdot a ), \cdots,  \tau(U \cdot \a )]) \\
        & \theta^{*\top} = \u^{*\top} + U \cdot \boldsymbol{m}, 
        \label{lambada_eq}
        \end{aligned}
        \end{equation}
        where $\mathtt{KD}$ denotes executing the KD-Tree algorithm, $U$ is the length of the codebook $[\tau(1 \cdot \a ), \cdots,  \tau(U \cdot \a )]$, and $\tau(\u^* \cdot \a )$ is the closest points to the scaled points $O$.

    \item When both $U$ and $G$ are large (e.g. $U=1600$ and $G=10^9$), the standard KD-Tree search shown in Eq.\eqref{lambada_eq} remains computationally expensive due to the massive search space. To further accelerate the compression process, we optimize the standard KD-Tree, instead of directly using it as the following formula:
    \begin{equation}
    \begin{aligned}
    &u^* = \mathtt{KD} (\w^{in}, [\tau(1 \cdot \a ), \cdots,  \tau(U \cdot \a )]) \\
    \end{aligned},
    \end{equation}
    where $\w^{in}=[w^{in}_1, w^{in}_2]$ is a single two-dimensional weight vector we need to compress. 

    We observe that a full-codebook search within the KD-Tree is redundant. Since the nodes in the codebook $[\tau(1 \cdot \a ), \cdots,  \tau(U \cdot \a )]$ are organized in a structured manner with continuous indices. Specifically, $U$ nodes are aligned sequentially from the bottom-left $[a, c]$ to the top-right $[b, d]$ within a bounding box $[a, b] \times [c, d]$. Therefore, we can rapidly pinpoint the indices of the nodes in the two columns adjacent to a query point $\w^{(in)}$ to serve as a search subset:

    \begin{equation}
    \begin{cases}
    & u_l = \lfloor \frac{w^{in}_1 - a}{l/(\sqrt{U}-1)}\rfloor \cdot \sqrt{U} + \lfloor \frac{w^{in}_2 - c}{l/(\sqrt{U}-1)}\rfloor + 1\\
    & u_r = (\lfloor \frac{w^{in}_1 - a}{l/(\sqrt{U}-1)}\rfloor+1) \cdot \sqrt{U} + \lfloor \frac{w^{in}_2 - c}{l/(\sqrt{U}-1)}\rfloor + 2.\\
    \label{pkd}
    \end{cases}
    \end{equation}

    Then, we only need to perform the nearest neighbor search over a small subset of the full codebook:
    \begin{equation}
        u^* = \mathtt{KD} (\w^{in}, [\tau(u_l \cdot \a ), \tau((u_l+1) \cdot \a ), \cdots,  \tau(u_r \cdot \a )]),
    \end{equation}
    which is much faster in terms of every search. 
    
    We compare the nearest neighbor search speed of the standard KD-Tree and the modified one. As illustrated in Figure \ref{search_compare}, we randomly generate matrix of $\mathbb{R}^{G\times 2}$, and we evaluate how much time it takes for two kinds of methods to compress this matrix row-by-row. $G$ increases from $10^6$ to $10^9$. The results demonstrate that the modified method consistently achieves a 5$\times$ speedup.}

    \begin{figure}[h]
    \centering
    \includegraphics[width=1\linewidth]{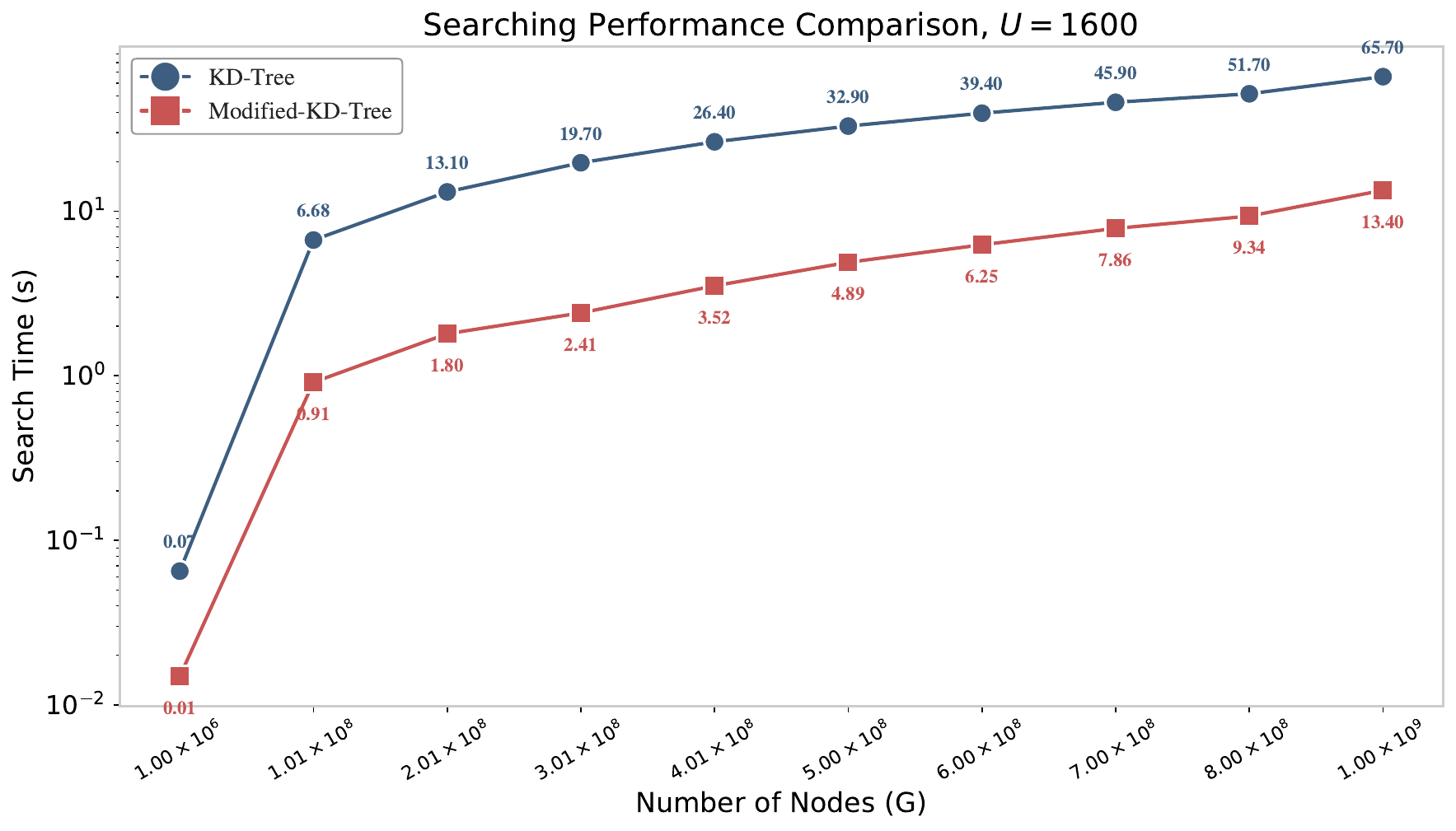} 
    \caption{\footnotesize {\color{black}Comparison of time comsumption in time with $U=1600$ as $G$ scales from $10^6$ to $10^9$. The results show a stable 5$\times$ performance gain after optimization. Even at the extensive scale of $10^9$ nodes, the search concludes in only 13.4 seconds.}} 
    \label{search_compare}
    \end{figure}

    \item We also multiprocess the compression by simultaneously compressing different layers of a model to make full use of computational resources. {\color{black}Regarding potential memory limits, our engineering implementation is explicitly designed to prevent out-of-memory issues when handling ultra-large models. As shown in Eq. \eqref{pkd}, the $\mathtt{KD}$ operation performs a nearest neighbor search independently. Leveraging this independence, we implement a batch processing mechanism for memory-constrained environments in parallel. Specifically, we can define a batch size to partition weights of the target model into multiple manageable subsets. The nearest neighbor search is then executed iteratively (batch-by-batch), and the partial results are finally concatenated together. This ensures that peak memory usage remains bounded, regardless of the total model size.}        
\end{itemize}

\definecolor{orange01}{HTML}{FC8D59} 
\definecolor{blue01}{HTML}{92BFDB}

\vspace{-0.1cm}

\section{Experiment and Analysis}


\begin{table}[h]
\begin{center}
\caption{The information of our used models as the testbed of different model compression methods.}
\scalebox{0.8}{
\begin{tabular}{+l ^c ^c}
\hline
\textbf{Model} & \textbf{Precision} & \textbf{File Size}\\
\hline
\rowstyle{\color{black}} ResNet18\cite{he2016deep} & FP32 & 44.60MB \\
\rowstyle{\color{black}} ResNet50\cite{he2016deep} & FP32 & 97.70MB \\
UNet \cite{ronneberger2015unet} & FP32 & 51.10MB \\
Pruned-UNet & FP32 & 20.20MB \\
MobileNetV3 \cite{howard2019searching} & FP32 & 5.95MB \\
Pruned-MobileNetV3 & FP32 & 2.22MB \\
LLaMA2-7B \cite{touvron2023llama} & FP16 & 12.50GB \\
Sheared-LLaMA-1.3B \cite{xia2024sheared} & FP32 & 5.01GB \\
TinyLLaMA \cite{zhang2024tinyllama} & FP32 & 4.09GB \\
LiteLLaMA \footnotemark[1] & FP16 & 0.86GB \\
\rowstyle{\color{black}} LLaMA-3.2-1B-Instruct \footnotemark[4]  & FP16 & 2.47GB \\
\rowstyle{\color{black}} Qwen1.5-7B-Chat \footnotemark[5] & FP16 & 15.45GB \\
\rowstyle{\color{black}} DeepSeek-R1-Distill-Qwen-14B \footnotemark[6] & FP16 & 29.54GB \\
\rowstyle{\color{black}} DeepSeek-R1-Distill-Qwen-32B \footnotemark[6] & FP16 & 65.54GB \\
\rowstyle{\color{black}} Qwen2-72B-Instruct \cite{team2024qwen2} & FP16 & 145.51GB \\
\hline
\end{tabular}}
\label{table_param}
\end{center}
\end{table}

\begin{table}[t!]
\caption{\color{black}{Detailed hyperparameter settings in Hyper-Compression for all the models.}}
\centering
\scalebox{0.75}{
\begin{tabular}{+l|^c|^c|^c}
\toprule
\textbf{Model} & Max class & \textbf{$U$} & \textbf{$l$}  \\
\midrule
\rowstyle{\color{black}} UNet & 3 & [225] & [0.1]  \\
\rowstyle{\color{black}} MobileNetV3 & 4 & [1600] & [0.8]  \\
\rowstyle{\color{black}} LLaMA2-7B & 3 & [1600] & [0.1]  \\
\rowstyle{\color{black}} Sheared-LLaMA-1.3B & 3 & [1600] & [0.1] \\
\rowstyle{\color{black}} TinyLLaMA & 3 & [1600] & [0.1] \\
\rowstyle{\color{black}} LiteLLaMA & 3 & [1600] & [0.1] \\
\rowstyle{\color{black}} DeepSeek-R1-Distill-QWen-14B & 3 & [1681] & [0.1]  \\
\rowstyle{\color{black}} DeepSeek-R1-Distill-QWen-32B & 3 & [1681] & [0.1]  \\ 
\rowstyle{\color{black}} DeepSeek-R1-Distill-QWen-72B & 3 & [900, 1296, 1521, 1681] & [0.1]  \\ 
\rowstyle{\color{black}} LLaMA-3.2-1B-Instruct & 3 & [1681] & [0.1]  \\
\rowstyle{\color{black}} Qwen1.5-7B-Chat & 3 & [900, 1296, 1521, 1681] & [0.05, 0.1, 0.15]  \\
\bottomrule
\end{tabular}
}
\label{hyper_table}
\vspace{-0.5cm}
\end{table}

This section illustrates the efficacy of our novel compression methodology on three widely utilized representative models: {\color{black}LLaMA~\cite{touvron2023llama} series, Qwen~\cite{ahmed2025qwen} series, UNet~\cite{ronneberger2015unet}, ResNet series~\cite{he2016deep}, and MobileNetV3~\cite{howard2019searching}, which correspond to  large, middle, and small models, respectively.} Table \ref{table_param} shows the information of our used models as the testbed of different model compression methods. Favorably, our compression technique does not require post-hoc retraining, even at high compression ratios. For example, we achieve up to a $7.87\times$ reduction in the model size of UNet within $1\%$ performance drop. This is a significant improvement over traditional methods that often require extensive retraining to restore efficacy. This capability sets a new benchmark in model compression. Because our proposed Hyper-Compression method requires no post-hoc training, we do not compare it with distillation-based methods that demands a great amount of data to train the student model.

\begin{table*}[t]
    \begin{center}
    \caption{{\color{black}Comparison between our methods and other compressed models on 8 downstream tasks. As a reference, INT4 quantization can achieve a maximum compression ratio of 4$\times$ on FP16 models.}}
    \scalebox{0.9}{
    \setlength{\tabcolsep}{3mm}{
    \begin{tabular}{+l ^c ^c ^c ^c ^c} 
    \hline 
    \textbf{Model} & \textbf{File Size (GB)} & \textbf{SciQ (\%)} & \textbf{WinoGrande (\%)} & \textbf{ARC-E (\%)} & \textbf{ARC-C (25) (\%)}\\
    \hline
    LLaMA2-7B \cite{touvron2023llama} & 12.50  & 94.00 & 68.98 & 76.30 & 52.39\\
    \rowstyle{\color{black}} DeepSeek-R1-Distill-QWen-14B \footnotemark[6] & 29.54 & 95.50 & 72.85 & 78.11 & 58.62\\
    \rowstyle{\color{black}} DeepSeek-R1-Distill-QWen-32B \footnotemark[6] & 65.54 & 95.90 & 75.69 & 81.06 & 63.74\\
    \rowstyle{\color{black}} Qwen2-72B-Instruct \cite{team2024qwen2} & 145.51 & 97.10 & 77.82 & 85.27 & 68.60\\
    OPT-1.3B & 2.63 & 84.30 & 59.60 & 57.00 & 29.70\\
    Pythia-1.4B & 2.93 & 86.40 & 57.40 & 60.70 & 31.20\\
    OPT-2.7B & 5.30 & 85.80 & 60.80 & 60.80 & 34.00\\
    Pythia-2.8B & 5.68 & 88.30 & 59.70 & 64.40 & 36.40\\
    INCITE-Base-3B & 5.69 & 90.70 & 63.50 & 67.70 & 40.20\\
    Open-LLaMA--3B-v1 & 6.85 & 91.30 & 61.50 & 67.60 & 39.60\\
    Open-LLaMA--3B-v2 & 6.85 & 91.80 & 63.50 & 66.50 & 39.00\\
    Sheared-LLaMA-1.3B \cite{xia2024sheared} & 5.01 (2.50$\times$) & 87.30  & 58.09 & 60.98 & 34.04\\
    TinyLLaMA \cite{zhang2024tinyllama} &  4.09 (3.06$\times$)  & 89.30 & 59.43 & 61.66 & 37.12\\
    LiteLLaMA \footnotemark[1] & 0.86 (14.53$\times$)  & 75.10 & 52.64 & 47.85 & 24.32\\

    LLaMA2-7B + HC & 4.80 (2.60$\times$)  & 93.70 & 69.77 & 75.59 & 53.24\\
    Sheared-LLaMA-1.3B + HC & 0.98 (12.76$\times$)  & 87.90 & 58.88 & 60.48 & 32.85\\
    TinyLLaMA + HC & 0.78 (16.03 $\times$) & 89.50 & 58.96 & 61.11 & 36.43\\
    LiteLLaMA + HC & 0.39 (32.05$\times$) & 73.00 & 54.22 & 44.78 & 23.89\\
    \rowstyle{\color{black}} DeepSeek-R1-Distill-QWen-14B + HC & 11.96 (2.47$\times$)  & 95.00 & 71.11 & 76.47 & 57.76\\
    \rowstyle{\color{black}} DeepSeek-R1-Distill-QWen-32B + HC & 26.53 (2.47$\times$)  & 96.00 & 75.85 & 80.18 & 63.57\\
    \rowstyle{\color{black}} Qwen2-72B-Instruct + HC & 50.18 (2.90$\times$) & 97.20 & 77.51 & 85.73 & 69.28\\
    
    \hline
    \hline
    
    \textbf{Model} & \textbf{HellaSwag (10) (\%)} & \textbf{BoolQ (32) (\%)} & \textbf{NQ (32) (\%)} & \textbf{MMLU (5) (\%)} & \textbf{Average (\%)}\\
    \hline
    LLaMA2-7B \cite{touvron2023llama} & 78.94  & 81.90 & 28.67 & 45.86 & 65.88\\
    \rowstyle{\color{black}} DeepSeek-R1-Distill-QWen-14B \footnotemark[6] & 81.05  & 88.69 & 22.22 & 74.62 & 71.46\\
    \rowstyle{\color{black}} DeepSeek-R1-Distill-QWen-32B \footnotemark[6] & 83.27 & 90.52 & 26.37 & 80.90 & 74.68\\
    \rowstyle{\color{black}} Qwen2-72B-Instruct \cite{team2024qwen2} & 86.95 & 89.88 & 36.59 & 83.91 & 78.27\\
    OPT-1.3B & 54.50  & 57.50 & 6.90 & 24.70 & 46.78\\
    Pythia-1.4B & 53.00  & 57.40 & 6.20 & 25.70 & 47.25\\
    OPT-2.7B & 61.50  & 63.40 & 10.10 & 25.90 & 50.29\\
    Pythia-2.8B & 60.80  & 66.00 & 9.00 & 26.90 & 51.44\\
    INCITE-Base-3B & 64.80  & 65.90 & 14.90 & 27.00 & 54.34\\
    Open-LLaMA--3B-v1 & 62.60  & 70.00 & 18.60 & 27.00 & 54.78\\
    Open-LLaMA--3B-v2 & 67.60  & 69.60 & 17.10 & 26.90 & 55.20\\
    
    Sheared-LLaMA-1.3B \cite{xia2024sheared} & 61.02  & 65.54 & 9.89 & 25.59 & 50.31\\
    TinyLLaMA \cite{zhang2024tinyllama} & 62.48  & 62.91 & 12.52 & 26.79 & 51.53\\
    LiteLLaMA & 38.41  & 57.09 & 1.77 & 26.11 & 40.41\\
    LLaMA2-7B + HC & 77.17  & 80.92 & 25.65 & 43.04 & 64.89\\
    Sheared-LLaMA-1.3B + HC & 60.56  & 63.76 & 8.98 & 24.68 & 49.76\\
    TinyLLaMA + HC & 61.92  & 58.41 & 11.58 & 27.28 & 50.65\\
    LiteLLaMA + HC & 37.56  & 56.57 & 1.16 & 26.61 & 39.72\\
    \rowstyle{\color{black}} DeepSeek-R1-Distill-QWen-14B + HC & 80.36  & 88.44 & 21.39 & 73.47 & 70.50\\
    \rowstyle{\color{black}} DeepSeek-R1-Distill-QWen-32B + HC & 82.94 & 90.34 & 25.18 & 80.43 & 74.31\\
    \rowstyle{\color{black}} Qwen2-72B-Instruct + HC & 86.85  & 89.85 & 36.29 & 83.73 & 78.31\\
    \hline
    
    \end{tabular}
    }}
    \label{table01}
    \end{center}
    \vspace{-0.5cm}
\end{table*}

\begin{table}[h!]
    \begin{center}
    \caption{The comparison of Perplexity (PPL) on the dataset \textit{wikitext-2-raw-v1}}
    \scalebox{1}{\begin{tabular}{ccc}
    \hline
    \textbf{Model} & \textbf{Rate} & \textbf{PPL} \\
    \hline
    LlaMA2-7B & - & 5.47 \\
    \makecell[c]{LlaMA2-7B +\\ W8A8 SmoothQuant \cite{xiao2023smoothquant}} & 2.00$\times$  & 5.52\\
    \makecell[c]{LlaMA2-7B +\\ Asymmetric Quant (int8)} & 2.00$\times$  & 5.65\\
    \makecell[c]{LlaMA2-7B +\\ GPTQ \cite{frantar2022gptq} (int4)} & 3.46$\times$  & 5.69\\
    \makecell[c]{LlaMA2-7B +\\ Asymmetric Quant (int4)} & 4.00$\times$  & 26160.34\\
    LlaMA2-7B + HC & 2.60$\times$  & 5.82\\
    Sheared-LlaMA-1.3B & 2.50$\times$  & 8.13\\
    Sheared-LlaMA-1.3B + HC & 12.76$\times$  & 8.37\\
    TinyLlaMA & 3.06$\times$  & 7.71\\
    TinyLlaMA + HC & 16.03$\times$  & 7.95\\
    LiteLlaMA & 14.53$\times$ & 31.82 \\
    LiteLlaMA + HC & 32.05$\times$  & 37.85\\
    \hline
    \end{tabular}}
    \label{table_ppl}
    \end{center}
    \vspace{-0.5cm}
\end{table}

\textbf{Hardware Configuration and Detailed Compression Parameters}. {\color{black}All experiments are conducted on a server equipped with NVIDIA RTX L40, NVIDIA RTX A40, NVIDIA RTX 4060  and NVIDIA RTX 6000 Pro GPUs. Given the extensive scope of our experiments, listing all hyperparameters in the main text would be exhaustive. Therefore, we have compiled a detailed Hyperparameter Table \ref{hyper_table}, which covers the specific settings for all reported results.}

\textbf{Preferable compression ratio}. As shown in Table \ref{table01}, we apply the same pruning method used in Sheared-LLaMA-1.3B, TinyLLaMA, and LiteLLaMA, combined with our compression technique. We evaluate our compressed models on eight downstream tasks: 0-shot accuracy on SciQ, WinoGrande, ARC-E, 25-shot ARC-C, 10-shot HellaSwag, 32-shot BoolQ, NQ, and 5-shot MMLU. The average test results of eight downstream tasks are shown in Table \ref{table01}, where “File Size" indicates the size of the file that stores the model. Notably, our method can compress LLaMA2-7B by a factor of $2.60\times$ while maintaining the average score decrease within 1\%, which achieves the best balance, while other models either achieve sup-optimal compression rates or bear a large performance drop. For example, Sheared-LLaMA-1.3B achieves $2.50\times$ with $15.57\%$ performance decreases, while TinyLLaMA achieves $3.60\times$ with $14.33\%$ performance loss on average. {\color{black}Furthermore, we test the compressed model and the original one on more downstream benchmarks. The results are shown in Table \ref{more_tasks}. We expand the evaluation to 9 new tasks: GSM8K\cite{cobbe2021training} for multi-step mathematical reasoning; PIQA\cite{bisk2020piqa} for physical commonsense understanding; and the SuperGLUE suite (CB\cite{de2019commitmentbank}, COPA\cite{roemmele2011choice}, MultiRC\cite{khashabi2018looking}, ReCoRD\cite{zhang2018record}, RTE\cite{sarlin2020superglue}, WiC\cite{pilehvar2019wic}, WSC\cite{levesque2012winograd}) covering linguistic inference, coreference resolution, and reading comprehension. These tasks collectively assess critical dimensions of language model performance, including logical reasoning, world knowledge, and contextual understanding, which can provide a more rigorous validation of model robustness. To the best of our knowledge, few studies in the domain of model compression tested as many as 16 downstream tasks as our work does. For instance, Sheared-LLaMA~\cite{xia2024sheared} evaluated 11 downstream tasks, and AWQ~\cite{lin2024awq} included 11 tasks in its comparative experiments; meanwhile, LLaMA~\cite{touvron2023llama} and Qwen2.5~\cite{ahmed2025qwen} were tested on only 8 and 5 tasks, respectively.} 

\begin{table*}[h!]
\begin{center}
\caption{{\color{black}Performance comparison of Qwen1.5-7B-Chat and its hyper-compressed variant across 16 benchmark tasks}}
\scalebox{0.82}{
\setlength{\tabcolsep}{3mm}{
\begin{tabular}{+r ^c ^c ^c ^c ^c ^c ^c ^c ^c} 
\hline 
\addlinespace
\textbf{Model} & \textbf{Ratio} & \textbf{SciQ } & \textbf{WinoGrande } & \textbf{ARC-E } & \textbf{ARC-C} & \textbf{HellaSwag} & \textbf{BoolQ} & \textbf{MMLU} & \textbf{GSM8K} \\
\addlinespace
\hline
\addlinespace
\rowstyle{\color{black}} Qwen1.5-7B-Chat & - & 92.60 & 65.43 & 68.48 & 52.56 & 78.65 & 85.44 & 60.48 & 54.06\\
\addlinespace
\rowstyle{\color{black}} \quad \textbf{+ Hyper-Compression} & 2.67 & 91.30 & 65.19 & 67.13 & 52.73 & 77.72 & 85.17 & 59.24 & 51.18\\
\addlinespace

\hline
\hline
\addlinespace
\textbf{PIQA} & \textbf{CB} & \textbf{COPA} & \textbf{MultiRC} & \textbf{ReCoRD} & \textbf{RTE} & \textbf{WiC} & \textbf{WSC} & \textbf{Average} & \textbf{Drop} \\
\addlinespace
\hline
\addlinespace
\rowstyle{\color{black}} 75.19 & 60.71 & 84.00 & 26.26 & 80.18 & 83.03 & 67.40 & 55.77 & 65.23 & -\\
\addlinespace
\rowstyle{\color{black}} 74.65 & 60.71 & 89.00 & 23.54 & 79.87 & 80.59 & 66.77 & 59.62 & 64.75 & 0.48\\
\addlinespace
\hline
\end{tabular}
}}
\label{more_tasks}
\end{center}
\vspace{-0.5cm}
\end{table*}

\begin{table}[h!]
    \centering
    \caption{{\color{black}The compression effectiveness of our method on ResNet18, ResNet50, UNet and MobileNetV3. Moreover, we also compare our method with other data-free methods on ResNet18 and ResNet50.}}
    \scalebox{0.9}{
    \begin{tabular}{+l ^c ^c}
    \hline
    \addlinespace[0.8ex]
    \textbf{Model} & \textbf{Ratio} & \textbf{Dice (\%) } \\
   \hline
    UNet \cite{ronneberger2015unet} & - & 99.86 \\
    Pruning & 2.53$\times$ & 96.34 \\
    HC & 7.87$\times$ & 99.71 \\
    Pruning + HC & 17.74$\times$ & 96.45 \\
    \hline
    \textbf{Model} & \textbf{Ratio} & \textbf{Top-1 Acc (\%)}\\
    \hline
    MobileNetV3 \cite{howard2019searching} & - & 74.41 \\
    Pruning & 2.68$\times$ & 69.32\\
    HC & 3.61$\times$ & 73.93 \\
    Pruning + HC & 12.66$\times$ & 68.47 \\
    \hline
    \textbf{Model} & \textbf{Ratio} & \textbf{Top-1 Acc (\%)}\\
    \hline

    \rowstyle{\color{black}} ResNet18 & - & 71.47\\
    \rowstyle{\color{black}} ResNet18-\textbf{ZeroQ-INT4} & 8$\times$ & 19.09\\
    \rowstyle{\color{black}} ResNet18-\textbf{ZeroQ-INT8} & 4$\times$ & 71.43\\
    \rowstyle{\color{black}} ResNet18-\textbf{GDFQ-INT4} & 8$\times$ & 60.60\\
    \rowstyle{\color{black}} ResNet18-\textbf{GDFQ-INT8} & 4$\times$ & 70.68\\
    \rowstyle{\color{black}} ResNet18-\textbf{SQuant-INT4} & 8$\times$ & 66.14\\
    \rowstyle{\color{black}} ResNet18-\textbf{SQuant-INT8} & 4$\times$ & 71.47\\
    \rowstyle{\color{black}} ResNet18-\textbf{HC} & 6.21$\times$ & 70.93\\
    \rowstyle{\color{black}} ResNet50 & - & 77.74\\
    \rowstyle{\color{black}} ResNet50-\textbf{ZeroQ-INT4} & 8$\times$ & 7.75\\
    \rowstyle{\color{black}} ResNet50-\textbf{ZeroQ-INT8} & 4$\times$ & 77.65\\
    \rowstyle{\color{black}} ResNet50-\textbf{GDFQ-INT4} & 8$\times$ & 55.65\\
    \rowstyle{\color{black}} ResNet50-\textbf{GDFQ-INT8} & 4$\times$ & 77.51\\
    \rowstyle{\color{black}} ResNet50-\textbf{SQuant-INT4} & 8$\times$ & 70.80\\
    \rowstyle{\color{black}} ResNet50-\textbf{SQuant-INT8} & 4$\times$ & 77.71\\
    \rowstyle{\color{black}} ResNet50-\textbf{HC} & 5.92$\times$ & 76.40\\

    \hline
    \end{tabular}}
    \label{table_unet}
    \vspace{-0.5cm}
\end{table}

Moreover, perplexity (PPL) is a vital metric for evaluating the overall performance of large language models (LLMs). We conduct tests on the publicly available dataset wikitext\footnotemark[2]. As shown in Table \ref{table_ppl}, the PPL of LLaMA2-7B+HC increases by only 0.35 compared to the original LLaMA2-7B model, whereas the PPL increases caused by other three LLaMA2-7B variants (Sheared-LlaMA-1.3B, TinyLlaMA, LiteLlaMA) are 2.66, 2.24, and 26.35, respectively. We also compare the perplexity with four quantization methods: W8A8 SmoothQuant (int8, $\alpha$=0.85), Asymmetric Quant (int8), GPTQ (int4), and Asymmetric Quant (int4). Among them, both W8A8 SmoothQuant (int8, $\alpha$=0.85) and GPTQ (int4) require calibration sets for calibration during the quantization process, whereas our method does not involve calibration. Without re-training, our model is already better than the int8 quantization but inferior to the int4 quantization with retraining. This demonstrates that our method is competitive in real-world applications. Figure \ref{talks} compares the text outputs generated by eight different LLM models for a given prompt.

{\color{black}
We also incorporate comparative experiments with four prominent quantization methods: AWQ \cite{lin2024awq}, GPTQ \cite{frantar2022gptq}, SmoothQuant \cite{xiao2023smoothquant}, and SpinQuant \cite{liu2024spinquant}. All four methods are Post-Training Quantization (PTQ) algorithms that even require calibration datasets for parameter optimization. We conduct these evaluations on the Llama-3.2-1B-Instruct and Qwen1.5-7B-Chat models. 

The evaluation results are reported in Tables \ref{compare_with_quant_A} and \ref{compare_with_quant_B}. We evaluate all the models on eight downstream tasks:  0-shot accuracy on SciQ, WinoGrande, ARC-E, 25-shot ARC-C, 10-shot HellaSwag, 32-shot BoolQ, NQ, and 5-shot MMLU. Notably, our method compresses LLaMA-3.2-1B-Instruct and Qwen1.5-7B-Chat by factors of $1.62\times$ and $2.67\times$, respectively, while limiting the average score decrease to within 1.24\% and 0.94\%. This achieves the optimal balance between efficiency and accuracy, whereas other models either yield sub-optimal compression rates or suffer from significant performance degradation. For instance, although GPTQ-INT8 restricts the average score drop in LLaMA-3.2-1B-Instruct within only 0.06\%, it achieves a limited compression factor of $1.69\times$. Similarly, while SmoothQuant-INT4 reaches a compression factor of $1.59\times$, it results in a substantial average performance drop of 5.14\%. Considering that the INT4 quantization is the widely-recognized frontier, we reasonably conclude that our method is competitive. It should be noted that, owing to the extensive experimental scope, all comparative methods were evaluated using officially released open-source implementations from \href{https://huggingface.co}{HuggingFace}. This approach eliminates potential discrepancies arising from manual reimplementation, thereby ensuring the fairness of model performance assessment.
}

\begin{table*}[h!]
\centering
\footnotesize
\sisetup{table-format=2.2, table-number-alignment=center}
\begin{threeparttable}
\caption{{\color{black}Comparison between our methods and other compressed models of LLaMA-3.2-1B-Instruct and Qwen1.5-7B-Chat on individual evaluation dataset (Part A). To facilitate comparison, the best results are highlighted in \textbf{bold}, and the second-best results are \underline{underlined}.}}
\label{tab:extended_metrics}
\begin{tabular}{
    @{}
    >{\RaggedRight}p{4.2cm} 
    *{5}{S}
    @{}
}
\toprule
\textbf{Model} & 
{\textbf{File Size(GB)}} & 
{\textbf{SciQ(\%)}} & 
{\textbf{WinoGrande(\%)}} & 
{\textbf{ARC-E(\%)}} & 
{\textbf{ARC-C(\%)}} \\
\midrule
LLaMA-3.2-1B-Instruct & 2.47 & 94.30 & 61.40 & 69.02 & 38.23 \\
\addlinespace[0.8ex]
\quad + SpinQuant-INT4\cite{liu2024spinquant} & {1.59 (1.55$\times$)} & 92.10 & 59.67 & 64.27 & 33.53 \\
\addlinespace[0.3ex]
\quad + AWQ-INT4\cite{lin2024awq} & {1.56 (1.58$\times$)} & \underline{93.80} & \textbf{62.27} & \underline{68.39} & 36.95 \\
\addlinespace[0.3ex]
\quad + SmoothQuant-INT4\cite{xiao2023smoothquant} & \underline{1.55 (1.59$\times$)} & 92.80 & 60.69 & 65.57 & 35.15 \\
\addlinespace[0.3ex]
\quad + SmoothQuant-INT8\cite{xiao2023smoothquant} & {2.02 (1.22$\times$)} & \textbf{94.40} & 61.09 & \textbf{69.15} & \textbf{38.31} \\
\addlinespace[0.3ex]
\quad + Hyper-Compression & \textbf{1.52 (1.62$\times$)} & 93.50 & \underline{62.19} & 67.42 & \underline{37.12} \\
\addlinespace[0.3ex]
\hline
\addlinespace[0.3ex]
Qwen1.5-7B-Chat & 15.45 & 92.60 & 65.43 & 68.48 & 52.56 \\
\addlinespace[0.8ex]
\quad + GPTQ-INT4\cite{frantar2022gptq} & \underline{5.86 (2.64$\times$)} & \underline{91.90} & 64.25 & 66.96 & 51.11 \\
\addlinespace[0.3ex]
\quad + GPTQ-INT8\cite{frantar2022gptq} & {9.12 (1.69$\times$)} & \textbf{92.80} & \underline{65.04} & \textbf{68.35} & \underline{52.47} \\
\addlinespace[0.3ex]
\quad + Hyper-Compression & \textbf{5.79 (2.67$\times$)} & 91.30 & \textbf{65.19} & \underline{67.13} & \textbf{52.73} \\
\bottomrule
\end{tabular}
\label{compare_with_quant_A}
\end{threeparttable}
\end{table*}

\begin{table*}[h!]
\centering
\footnotesize
\sisetup{table-format=2.2, table-number-alignment=center}
\begin{threeparttable}
\caption{{\color{black}Comparison between the proposed methods and other compressed Models of LLaMA-3.2-1B-Instruct and Qwen1.5-7B-Chat on individual evaluation dataset (Part B). To facilitate comparison, the best results are highlighted in \textbf{bold}, and the second-best results are \underline{underlined}.}}
\label{tab:extended_metrics}
\begin{tabular}{
    @{}
    >{\RaggedRight}p{4.2cm} 
    *{5}{S}
    @{}
}
\toprule
\textbf{Model} & 
{\textbf{Hellaswag(\%)}} & 
{\textbf{BoolQ(\%)}} & 
{\textbf{NQ(\%)}} & 
{\textbf{MMLU(\%)}} & 
{\textbf{Average(\%)}} \\
\midrule
LLaMA-3.2-1B-Instruct & 61.10 & 70.61 & 16.09 & 46.12 & 57.11 \\
\addlinespace[0.8ex]
\quad + SpinQuant-INT4\cite{liu2024spinquant} & 52.17 & 62.72 & 7.87 & 37.07 & 51.18 \\
\addlinespace[0.3ex]
\quad + AWQ-INT4\cite{lin2024awq} & 59.11 & \underline{69.42} & 12.41 & 43.36 & 55.71 \\
\addlinespace[0.3ex]
\quad + SmoothQuant-INT4 \cite{xiao2023smoothquant} & 52.84 & 63.09 & 8.12 & 37.51 & 51.97 \\
\addlinespace[0.3ex]
\quad + SmoothQuant-INT8 \cite{xiao2023smoothquant} & \textbf{61.08} & \textbf{70.52} & \textbf{15.73} & \textbf{45.90} & \textbf{57.02} \\
\addlinespace[0.3ex]
\quad + Hyper-Compression & \underline{59.59} & 68.78 & \underline{14.76} & \underline{43.63} & \underline{55.87} \\
\addlinespace[0.3ex] 
\hline
\addlinespace[0.3ex]
Qwen1.5-7B-Chat & 78.65 & 85.44 & 18.67 & 60.48 & 65.29 \\
\addlinespace[0.8ex]
\quad + GPTQ-INT4\cite{frantar2022gptq} & 77.69 & 84.71 & 16.15 & \underline{60.16} & 64.12 \\
\addlinespace[0.3ex]
\quad + GPTQ-INT8\cite{frantar2022gptq} & \textbf{78.57} & \textbf{85.44} & \textbf{18.64} & \textbf{60.52} & \textbf{65.23} \\
\addlinespace[0.3ex]
\quad + Hyper-Compression & \underline{77.72} & \underline{85.17} & \underline{16.34} & 59.24 & \underline{64.35} \\
\bottomrule
\end{tabular}
\label{compare_with_quant_B}
\end{threeparttable}
\end{table*}

As for UNet, MobileNetV3 and {\color{black}ResNet series}, \textbf{our method can compress UNet and Pruned-UNet by $7.87\times$ and $7.05\times$ with the performance loss contained in 1\%}, as shown in Table \ref{table_unet}. Particularly, our method succeeds in combination with other model compression methods such as pruning to achieve an even higher compression ratio. In UNet, \textbf{the total compression ratio is $17.74\times$ with the performance loss $3.41\%$}. {\color{black}Furthermore, we also compare our method with other zero-shot and minimal-data PTQ methods, which allows us to better assess the practical advantages. We first do the comparison in the \textbf{zero-shot} setting. We conduct experiments on the ImageNet validation dataset\cite{krizhevsky2012imagenet} using ResNet-18 and ResNet-50 models against three baselines: ZeroQ\cite{cai2020zeroq}, GDFQ\cite{xu2020generative}, and SQuant\cite{guo2022squant}. As shown in Table \ref{table_unet}, our method achieves compression ratios of 6.21$\times$ and 5.92$\times$ for ResNet-18 and ResNet-50 \cite{he2016deep}, respectively. Notably, these ratios surpass the theoretical maximum of INT8 quantization, although they remain lower than INT4. Moreover, the resulting performance degradation is 0.54\% and 1.34\%, respectively. This accuracy is significantly superior to all INT4 methods and only entails a slightly larger drop compared to INT8. For instance, while ResNet50-SQuant-INT8 incurs a negligible drop of 0.03\%, ResNet50-SQuant-INT4 suffers a severe drop of 6.94\%. This demonstrates that our method achieves a better trade-off between high compression ratios and model performance preservation.}

\textbf{No post-hoc retraining}. Table \ref{data_tokens} compares the number of tokens different compression methods use to retrain the compressed models, which highlights that our Hyper-Compression is zero-shot. This is because discrepancies between the original and decoded parameters in the Hyper-Compression are minor, at magnitudes ranging from $10^{-4}$ to $10^{-3}$. Then, the impact on error accumulation through layer-by-layer propagation is acceptable. This advantage is particularly useful in industry, where often 1) training data are inaccessible, and$/$or curating data is costly; 2) no computing resources are supplied for retraining. {\color{black}To bridge theoretical insights with practical feasibility, we empirically analyze error propagation in deep models. We conduct next-token prediction experiments using four representative models: LLaMA3.2-1B-Instruct, Qwen1.5-7B-Chat, LLaMA2-7B, and DeepSeek-R1-Distill-Qwen-14B and with four fixed input tokens (‘\textit{hyper}', ‘\textit{compression}', ‘\textit{is}', ‘\textit{magic}'). For all models, we compute the Mean Absolute Error (MAE) between the layer-wise outputs of the original and compressed variants across the first 100, 200, 200, and 300 linear layers of  LLaMA3.2-1B-Instruct, Qwen1.5-7B-Chat, LLaMA2-7B, and DeepSeek-R1-Distill-Qwen-14B, respectively. The empirical results provide preliminary validation of our theory, as our theory only describe the error created in a single layer. 

As illustrated in Figure \ref{error_accumulate}, the error accumulation surprisingly exhibits bounded fluctuation without demonstrating progressive accumulation across the network depth, indicating that our compression algorithm enjoys a stable error propagation pattern. We think this is because the trajectory of a dynamic system uniformly fills the entire space. Thus, given the weights, the compression error should also conform to an unbiased distribution, which leads to a canceling effect in Figure \ref{error_accumulate}. As our current work highlights the feasibility of the proposed dynamic system-based model compression, we leave the full explanation of this phenomenon to our future work.}

\begin{figure*}[t] 
    \centering
    \begin{subfigure}[b]{0.40\textwidth}
        \centering
        \includegraphics[width=\linewidth]{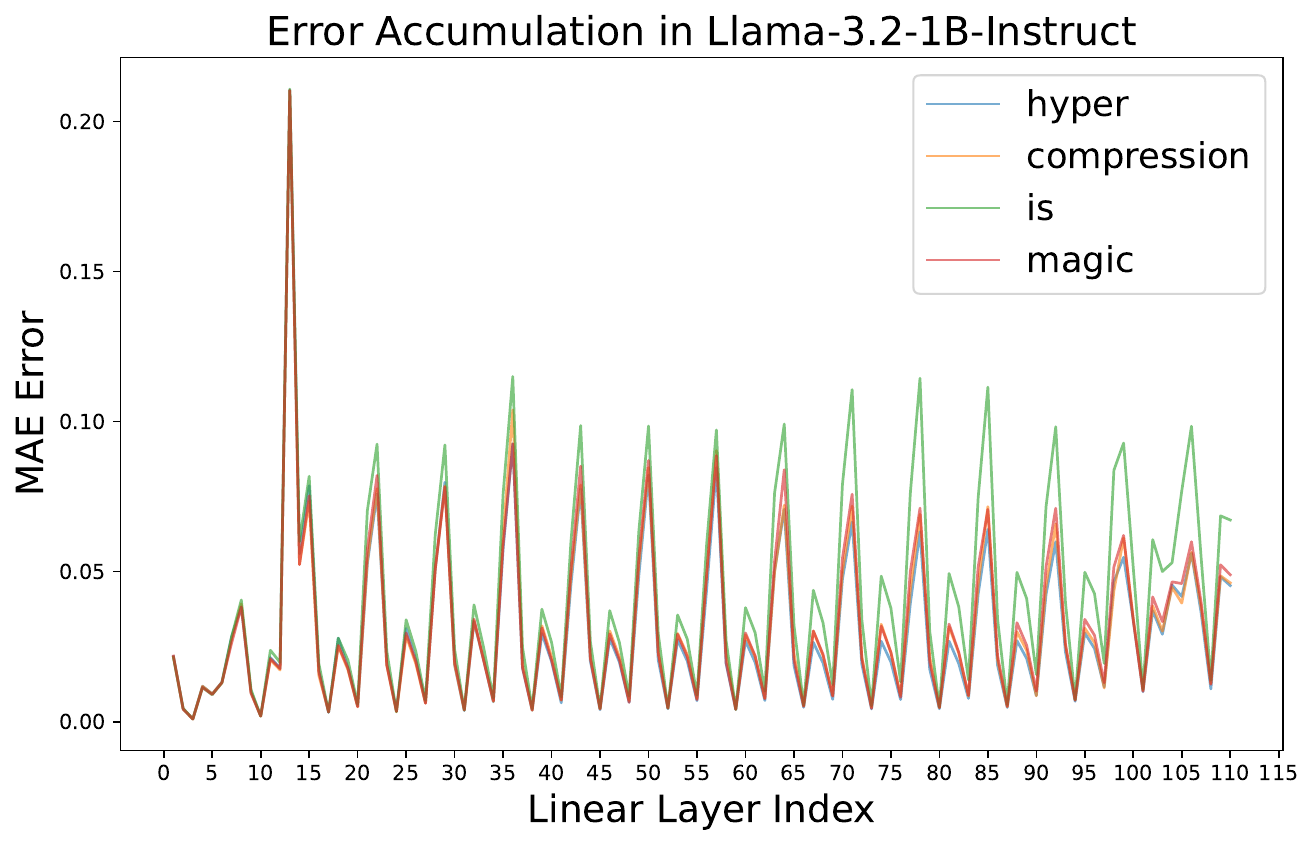}
    \end{subfigure}
    \begin{subfigure}[b]{0.40\textwidth}
        \centering
        \includegraphics[width=\linewidth]{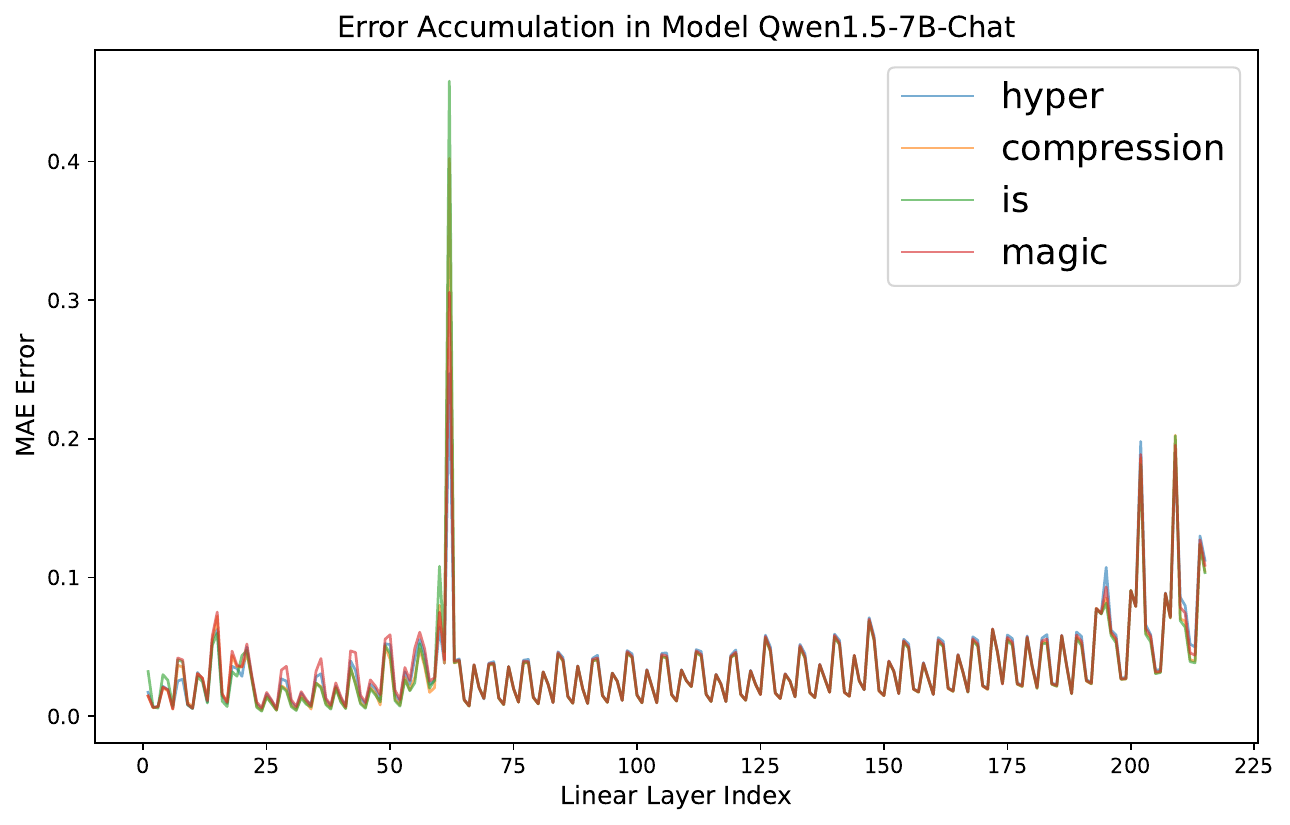}
    \end{subfigure}
    
    \vspace{1em} 
    
    \begin{subfigure}[b]{0.40\textwidth}
        \centering
        \includegraphics[width=\linewidth]{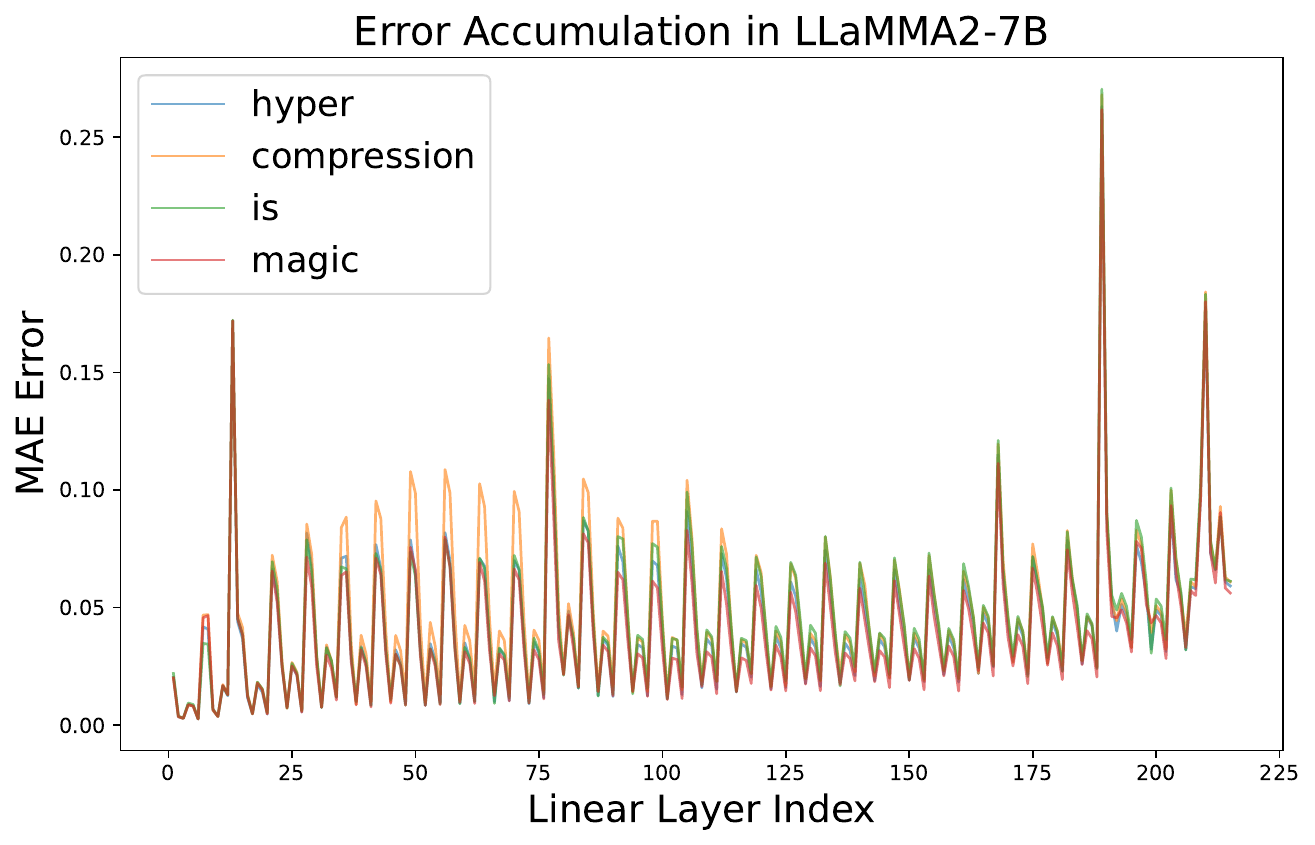}
    \end{subfigure}
    \begin{subfigure}[b]{0.40\textwidth}
        \centering
        \includegraphics[width=\linewidth]{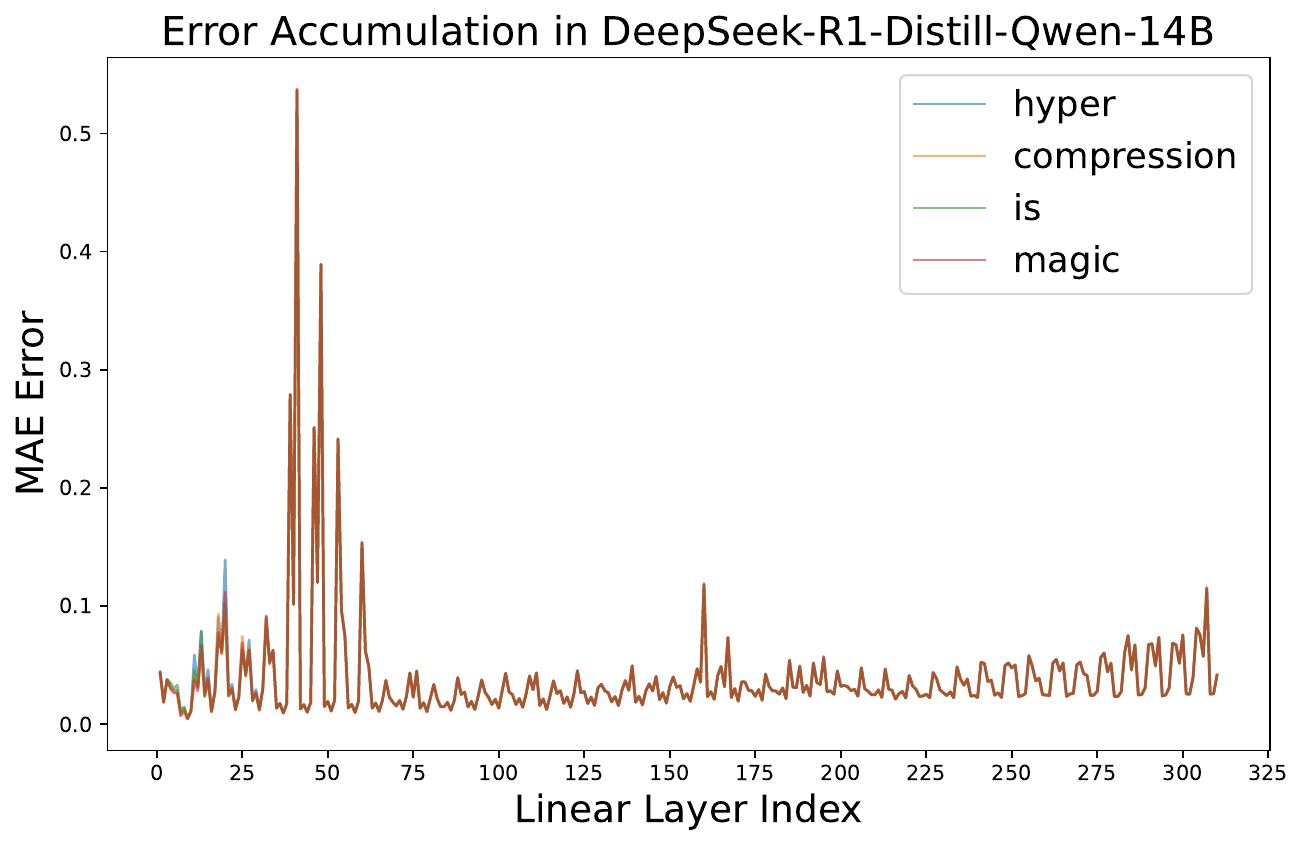}
    \end{subfigure}

    \caption{{\color{black}Layer-wise MAE propagation during next-token prediction on four compressed LLMs (LLaMA3.2-1B, DeepSeek-14B, LLaMA2-7B, Qwen1.5-7B) with fixed tokens (‘\textit{hyper}', ‘\textit{compression}', ‘\textit{is}', ‘\textit{magic}').}}
    \label{error_accumulate}
\end{figure*}

\textbf{Affordable inference time}. A network is a hierarchical structure. To expedite the inference time, our approach is to leverage this hierarchy by parallelizing the decoding and inference processes. The key is to complete the decoding of a layer before using this layer to infer. As shown in Table \ref{table02}, “Original" refers to the time required for a single inference of the original UNet, and “Ours" refers to the time required for one inference by employing parallel decoding and inference processes. It can be seen that the Hyper-Compression only increases the total inference time moderately.

\textbf{Short compression time}.
As Table \ref{table_comtime} shows, our method can compress models very fast. This efficiency primarily stems from using matrix operations for most compression processes and the KD-Tree to expedite the search. Additionally, by treating the compression tasks between layers as independent operations, we implement a parallel strategy to further decrease the compression time. 

\begin{table}[h!]
    \caption{The compression time of our method on UNet, MobileNetV3, LLaMA2-7B, DeepSeek-R1-Distill-Qwen-14B, DeepSeek-R1-Distill-Qwen-32B and Qwen2-72B-Instruct.}
    \begin{center}
    \scalebox{1}{\begin{tabular}{+c ^c}
    \hline
    \textbf{Model} & \textbf{Time (s)} \\
    \hline
    UNet + HC & 30.00 \\
    MobileNetV3 + HC & 11.63\\
    \hline
    \textbf{Model} & \textbf{Time (min)} \\
    \hline
    LlaMA2-7B + HC & 3.75 \\
    \rowstyle{\color{black}} DeepSeek-R1-Distill-Qwen-14B + HC & 7.79 \\
    \rowstyle{\color{black}} DeepSeek-R1-Distill-Qwen-32B + HC & 17.05 \\
    \rowstyle{\color{black}} Qwen2-72B-Instruct + HC & 33.60 \\
    \hline
    \end{tabular}}
    \label{table_comtime}
    \end{center}
    \vspace{-0.5cm}
\end{table}


\vspace{-0.1cm}

\section{Ablation Study and Parameter Sensitivity}
As shown in Table \ref{ablation01}, we conduct ablation experiments based on UNet and MobileNetV3 on the aforementioned two acceleration techniques to evaluate their independent contributions in computational speed. It is seen that both KD-Tree and matrix operations leads to a substantial enhancement in computational efficiency. Notably, the KD-Tree technique becomes more pronounced when the matrix operations are applied; and vice versa.

\begin{figure}[htb]
    \begin{center}
    \includegraphics[width=0.9\linewidth]{figures/talks.pdf}
    \caption{The exemplary outputs from LlaMA2-7B, Sheared-LlaMA-1.3B, TinyLlaMA, LiteLlaMA, and their compressed model by using Hyper-Compression. These results demonstrate that our compression method effectively preserves the models' ability to produce meaningful and grammatically correct text.}
    \label{talks}
    \end{center}
    \vspace{-0.3cm}
\end{figure}


\begin{table}[htb]
\caption{The results of ablation study on acceleration techniques.}
\centering
\scalebox{1}{\begin{tabular}{cccc}
\hline
\multirow{2}{*}{Model} &   K-D & Matrix & Compression\\
& Tree & Operations & Time (min) \\
\hline
\multirow{5}{*}{UNet \cite{ronneberger2015unet}}  & \ding{53} & \ding{53} & 231.6 \\
 & \checkmark & \ding{53} & 176.9 \\
 & \ding{53} & \checkmark & 54.8 \\
 & \checkmark & \checkmark & \textbf{0.5} \\
\hline
\multirow{5}{*}{MobileNetV3 \cite{howard2019searching}}  & \ding{53} & \ding{53} &  52.29 \\
 & \checkmark & \ding{53} & 3.13 \\
 & \ding{53} & \checkmark & 45.50 \\
 & \checkmark & \checkmark & \textbf{0.19} \\
\hline
\end{tabular}}
\label{ablation01}
\end{table}

\begin{figure}[h!] 
    \centering
    \begin{subfigure}[b]{0.43\linewidth}
        \centering
        \includegraphics[width=\linewidth]{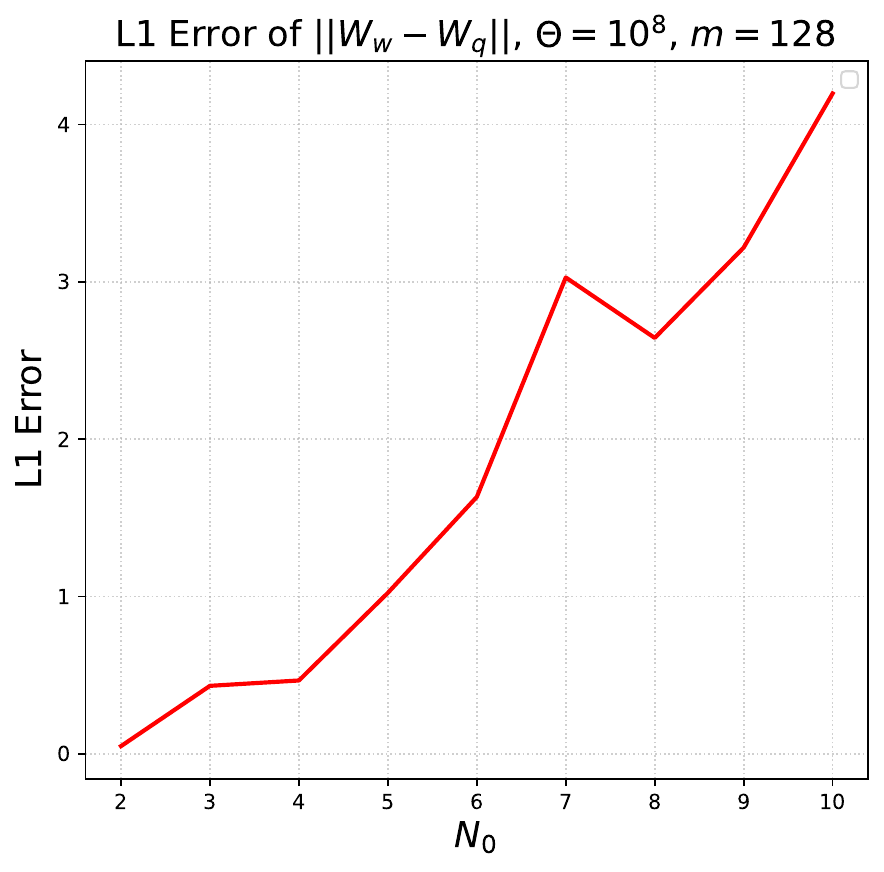}
    \end{subfigure}
    \begin{subfigure}[b]{0.43\linewidth}
        \centering
        \includegraphics[width=\linewidth]{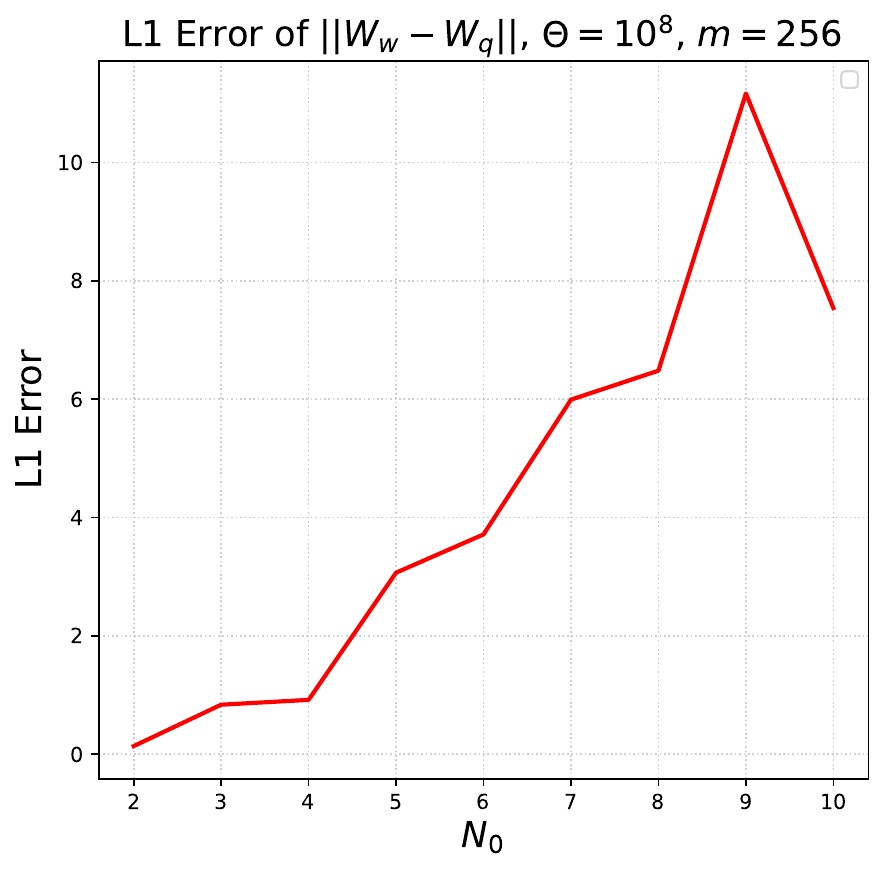}
    \end{subfigure}
    
    \vspace{0.5em} 
    
    \begin{subfigure}[b]{0.43\linewidth}
        \centering
        \includegraphics[width=\linewidth]{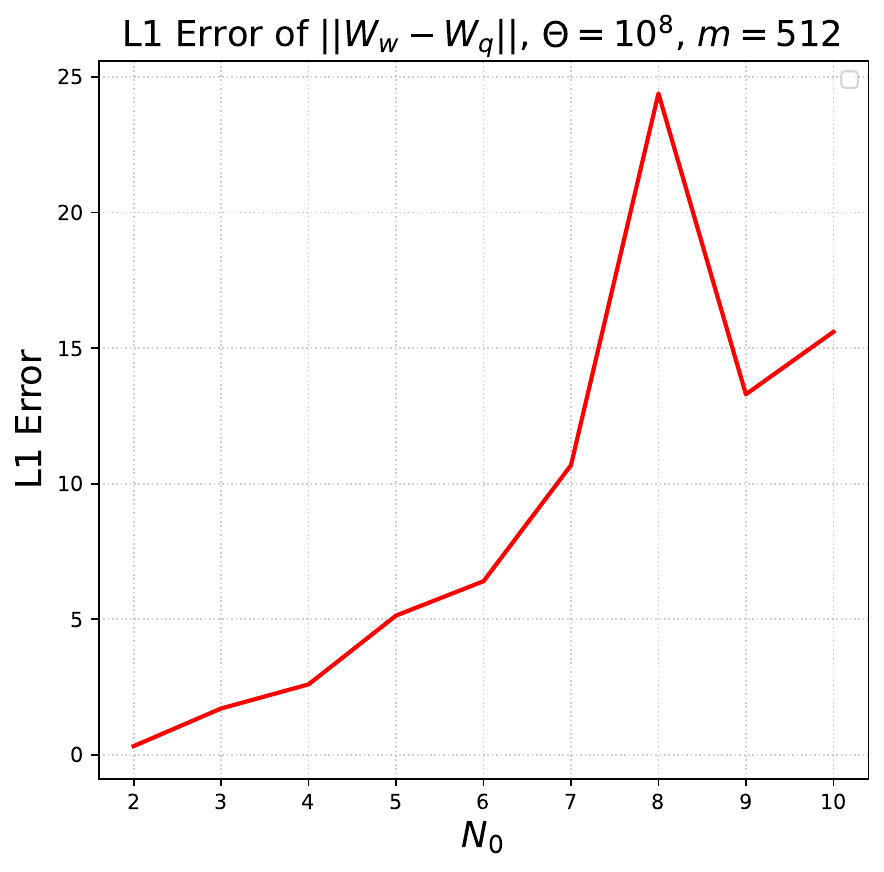}
    \end{subfigure}
    \begin{subfigure}[b]{0.43\linewidth}
        \centering
        \includegraphics[width=\linewidth]{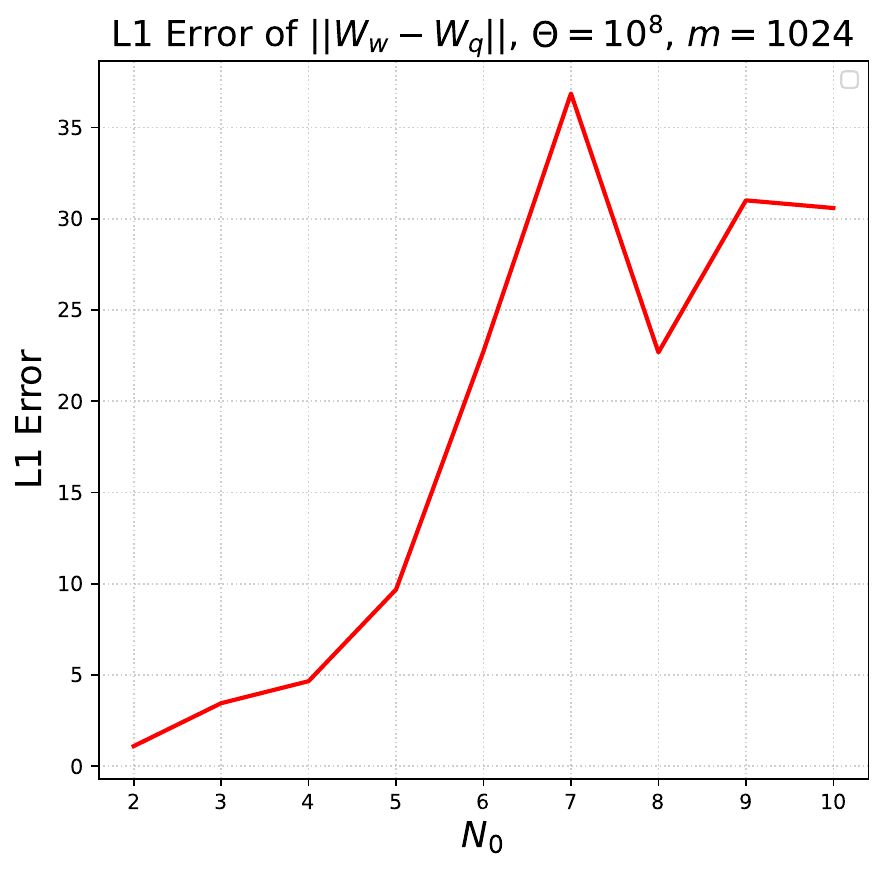}
    \end{subfigure}

    \caption{{\color
{blue}The experiments on synthetic weight matrices under controlled settings show that as $N_0$ increases, the compression error increases.}}
    \label{error_bound}
    \vspace{-0.5cm}
\end{figure}

We use UNet as an example to conduct a parameter sensitivity test on three aforementioned hyperparameters: $M$, $U$, and $l$. $M$ means the maximum value of categories that is set for all layers of the model, $U$ means the list of the number of sample nodes in the center square, and $l$ means the length of the side of the center square. As shown in Table \ref{tab:sensitivity}, different values of $M$, $U$, and $l$ only moderately impact the model's performance and compression ratio, which means our algorithm is robust to hyperparameters. 

\begin{table}[h!]
    \centering
    \caption{Parameter sensitivity results for different hyper-parameters on UNet.}
    \scalebox{0.9}{
    \setlength{\tabcolsep}{3mm}{
    \begin{tabular}{ccccc}
    \hline
     \multicolumn{3}{c}{\textbf{Hyper-Parameters}} & \multirow{2}{*}{\textbf{Dice ($\%$)}} & \multirow{2}{*}{\textbf{Rate}} \\
    \cline{1-3}
      Max Class & $U$ & $l$ \\
     \hline
     -  &-  &-  &99.86 &- \\
     1  &225  &0.02  &99.81 &7.09$\times$ \\
    2  &225  &0.02  &97.07 & 7.08$\times$ \\
     3  &225  &0.02  &90.11 & 7.08$\times$ \\
     \hline
     1  &225  &0.1  &99.21 & 7.92$\times$ \\
     2  &225  &0.1  &96.46 & 7.88$\times$ \\
     \cellcolor{blue01!60}\textbf{3}  &\cellcolor{blue01!60}\textbf{225}  & \cellcolor{blue01!60}\textbf{0.1}  &\cellcolor{blue01!60}\textbf{99.71}  & \cellcolor{blue01!60}\textbf{7.88}$\times$\\
    \hline
     1  &361  &0.02  &99.87 &6.38$\times$ \\
     2  &361  &0.02  &97.90 & 6.38$\times$ \\
     3  &361  &0.02  &99.39 & 6.38$\times$ \\
     \hline
     1  &361  &0.1  &99.28 & 7.73$\times$ \\
     2  &361  &0.1  &99.10 & 7.71$\times$ \\
     3  &361  &0.1  &99.65  & 7.68$\times$ \\
\hline
     1  &225 / 361  &0.02  &99.89 &7.08$\times$ \\
     2  &225 / 361  &0.02  &98.12 & 7.08$\times$ \\
     3  &225 / 361  &0.02  &99.58 & 7.08$\times$ \\
     \hline
     1  &225 / 361  &0.1  &99.08 & 7.73$\times$ \\
     2  &225 / 361  &0.1  &98.85 & 7.71$\times$ \\
     3  &225 / 361  &0.1  &99.51  &7.68$\times$\\
    
    \hline
    \end{tabular}}}
    \label{tab:sensitivity}
    \vspace{-0.2cm}
\end{table}

\begin{table}[h]
    \caption{The data tokens used in pre-training process compared with other models. Our model compression method don't need any pretraining or fine-tuning data.}
    \centering
    \scalebox{0.9}{\begin{tabular}{ccc}
    \hline
    \addlinespace[0.8ex]
    \textbf{Model} & \textbf{Tokens} & \textbf{Data Composition} \\
    \addlinespace[0.8ex]
    \hline
    \addlinespace[0.4ex]
    LLaMA2-7B~\cite{touvron2023llama} & 2T & \emph{Unknown} \\
    \addlinespace[0.4ex]
    \hline
    \addlinespace[0.4ex]
    Sheared-LLaMA-1.3B~\cite{xia2024sheared} & 50B & RedPajama \\
    \addlinespace[0.4ex]
    \hline
    \addlinespace[0.4ex]
    \multirow{3}{*}{TinyLLaMA~\cite{zhang2024tinyllama}} & \multirow{3}{*}{3T} & Slimpajama \\
    &  & + \\
    & & StarCoder \\
    \addlinespace[0.4ex]
    \hline
    \addlinespace[0.4ex]
    LiteLLaMA \footnotemark[1] & 1T & RedPajama \\
    \addlinespace[0.4ex]
    \hline
    \addlinespace[0.4ex]
    \multirow{5}{*}{OPT~\cite{zhang2022opt}} & \multirow{5}{*}{300B} & RoBERTa \\
    & & + \\
    & & The Pile \\
    & & + \\
    & & PushShift.io Reddit \\
    \addlinespace[0.4ex]
    \hline
    \addlinespace[0.4ex]
    Pythia~\cite{biderman2023pythia} & 300B & The Pile \\
    \hline
    \addlinespace[0.4ex]
    INCITE-Base~\footnotemark[2] & 800B & RedPajama \\
    \addlinespace[0.4ex]
    \hline
    \addlinespace[0.4ex]
    OpenLLaMA v1~\cite{geng2023openllama} & 1T & RedPaJama \\
    \addlinespace[0.4ex]
    \hline
    \addlinespace[0.4ex]
    \multirow{5}{*}{OpenLLaMA v2~\cite{geng2023openllama}} & \multirow{5}{*}{1T} & Falcon refined-web \\
    & & + \\
    & & StarCoder \\
    & & + \\
    & & Parts of RedPaJama \\
    \addlinespace[0.4ex]
    \hline
    \addlinespace[0.4ex]
    \bf{Ours} & \bf{0} & \bf{-} \\
    \addlinespace[0.4ex]
    \hline
    \end{tabular}}
    \label{data_tokens}
    \vspace{-0.25cm}
\end{table}

\begin{table}[htb]
    \begin{center}
    \caption{Timing results for different batch sizes on two devices using two methods with a UNet network.}
    \scalebox{0.9}{
    \begin{tabular}{ccccc}
    \hline
    \multirow{2}{*}{\textbf{Batch Size}} & \multicolumn{2}{c}{\textbf{4060}} & \multicolumn{2}{c}{\textbf{A40}} \\
    \cline{2-5}
     & \textbf{Original (s)} & \textbf{Ours (s)} & \textbf{Original (s)} & \textbf{Ours (s)} \\
     \hline
    1 &4.35  &4.43  &1.50  &2.06  \\
    2 &3.95  &4.49  &1.44  &2.05  \\
    4 &3.63  &4.77  &1.45  &2.22  \\
    8 &3.67  &7.94  &1.44  &2.46  \\
    \hline
    \end{tabular}}
    \label{table02}
    \end{center}
    \vspace{-0.5cm}
\end{table}

{\color{black}
\section{Discussion}
\textbf{The choice of hyperfunction for Hyper-Compression}. We choose the irrational winding on the torus as our hyperfunction because of its high computational efficiency. The irrational winding as a candidate hyperfunction is extremely simple and efficient to compute: $w_n = \tau(\theta\cdot\alpha_n)$, where $\tau(z)=z-\lfloor z \rfloor$, and $\alpha_n$ is irrational. This means that to recover weights of the original network, we only need a single multiplication and modulo operation per weight, which is instrumental in reducing the delay in inference. Meanwhile, finding $\theta$ does not need to deal with a partial differential equation (PDE). Thus, we can leverage matrix operations to compress and decompress entire layers in parallel, thereby allowing our method to scale to multi-billion parameter models, \textit{e.g.}, compressing LLaMA2-70B in 33.6 minutes (please see Table \ref{table_comtime}).

Our work establishes a relationship between the ergodicity of the dynamic system and model compression. Other famous ergodic dynamic systems should also have the potential of compression: i) Chaotic Maps \cite{alligood1998chaos}: While they can produce dense trajectories, they often exhibit sensitivity to initial conditions, rendering it hard to determine the initial point. Moreover, chaotic equations or formulas are serial, which need a long time to recover the weights, creating a delay in network inference. ii) Linear Congruential Generators (LCGs, \cite{boyar1989inferring}): They are also serial systems, and are hard to parallelize, facing the same issue. iii) Quasi-Monte Carlo Sequences \cite{caflisch1998monte}: systems with faster mixing rates could potentially achieve lower reconstruction errors for the same trajectory length. However, they typically require storing a base sequence or complex generation rules, making them hard to scale to large models.

We acknowledge that other dynamic systems might offer advantages in specific scenarios and are an exciting direction for future work. However, it can be seen that these systems face the challenge of scalability, which is the key matter in the era of large models. In contrast, the irrational winding strikes a good balance among strong scalability and competitive performance across diverse models.

\textbf{Discussion of theory}. Theorem \ref{theorem_4} provides a probabilistic upper bound on the approximation error for a single linear layer. This bound is derived using concentration inequalities and the properties of irrational winding. We emphasize that this bound is 
\begin{itemize}
    \item \textit{Asymptotic}: It describes the dynamic behavior of the compression error as $N_0$ changes.
    \item \textit{Worst-case}: It holds for the worst-case input distribution (uniform on the sphere) and does not fully account for the structure in real data.
\end{itemize}
In practice, we observe that the actual errors are significantly lower than the theoretical bound. This is because we employ engineering techniques like scaling and KD-tree-based trajectory assignment that can minimize the reconstruction error by adapting the hyperfunction to the weight distribution, particularly dealing with outliers appropriately.

However, we think that our theoretical derivation is meaningful in terms of informing how the key factors determine the recovering error. To evaluate the role of $N_0$, we conduct experiments on synthetic weight tensors under controlled settings: We fix $\Theta = 10^8$ and sample input data $\mathbf{X} \in [0,1]^{m\times N_0}$ from a uniform distribution. For each row of $\mathbf{X}$, we find the corresponding $\theta$. Four configurations are evaluated with $m \in \{128, 256, 512, 1024\}$. For each $N_0$, we compute the mean L1 error across 100 independent trials. As demonstrated in Figure \ref{error_bound}, the L1 error roughly exhibits an increasing trend with a growing $N_0$, which agrees with our theoretical prediction. 


}

\section{Conclusion and Future Work}
In this study, we have proposed Hyper-Compression, a novel and general-purpose methodology for model compression that leverages the trajectory density of some dynamic system to encode the parameters of the target network. We have conducted comprehensive and systematic experiments on various architectures, including LlaMA series, Qwen series, ResNet series, UNet, and MobileNetV3, demonstrating that our method is both user-friendly and scalable. Nevertheless, it represents merely an initial foray into the whole landscape. Future directions include {\color{black}1) developing a fused decode-compute method that integrates weight decoding directly into the matrix multiplication kernel for inference speedup}; 2) beyond classical theory, exploring other possibilities such as implicit neural representation \cite{chitturi2023capturing} to strike a better balance between compression time, inference time, and compression ratio; 3) considering distributions of weights of the target network to secure a higher compression rate. We believe that Hyper-Compression can contribute to Moore's law of model compression in the near future, \textit{i.e.}, the compression efficiency can be doubled annually, as a solution for the stagnation of hardware Moore's law.


\vspace{-0.45cm}
\renewcommand\refname{Reference}





%

\bibliographystyle{ieeetr}
\bibliography{pnas-sample.bib}

\vspace{-0.9cm}
\begin{IEEEbiography}[{\includegraphics[width=1in,height=1.25in,clip,keepaspectratio]{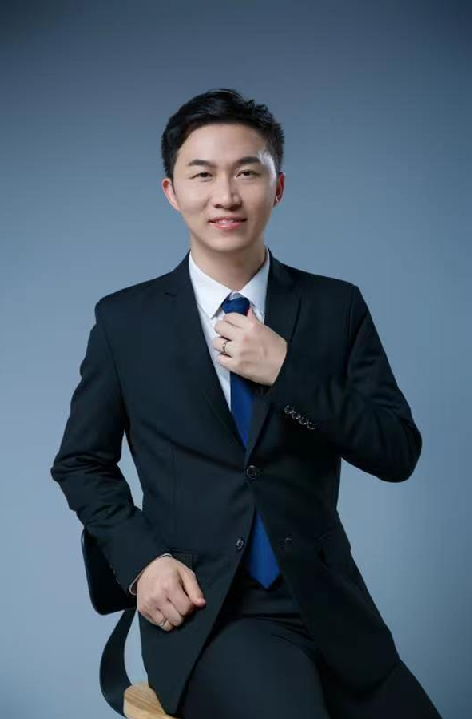}}]{Fenglei Fan} is currently an Assistant Professor with Department of Data Science, City University of Hong Kong. His primary research interests lie in NeuroAI and its applications in model compression and medical imaging. He was the recipient of the IBM AI Horizon Scholarship. His PhD dissertation was also selected as the recipient of the 2021 International Neural Network Society Doctoral Dissertation Award. His paper was selected as one of few 2024 CVPR Best Paper Award Candidates. He also won the IEEE Nuclear and Plasma Society IEEE TRPMS Best Paper Award. 
\end{IEEEbiography}

\vspace{-0.9cm}
\begin{IEEEbiography}[{\includegraphics[width=1in,height=1.25in,clip,keepaspectratio]{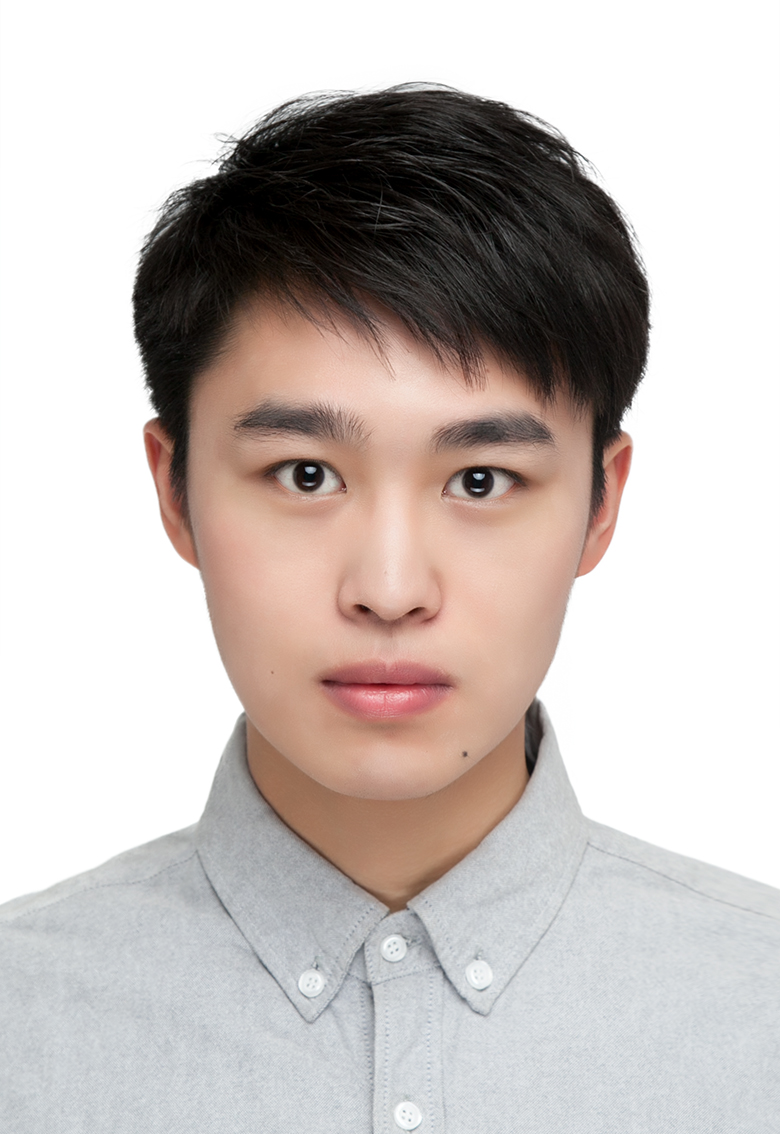}}]{Juntong Fan} is a Ph.D. student in the Department of Data Science at City University of Hong Kong. He holds dual Master's degrees: one in Mathematics from The Chinese University of Hong Kong and another in Social Computing from Xi'an Jiaotong-Liverpool University. His research primarily focuses on model compression theory and methodology. Prior to his doctoral studies, he served as a Research Assistant at CUHK from 2023 to 2024.
\end{IEEEbiography}

\vspace{-0.9cm}
\begin{IEEEbiography}[{\includegraphics[width=1in,height=1.25in,clip,keepaspectratio]{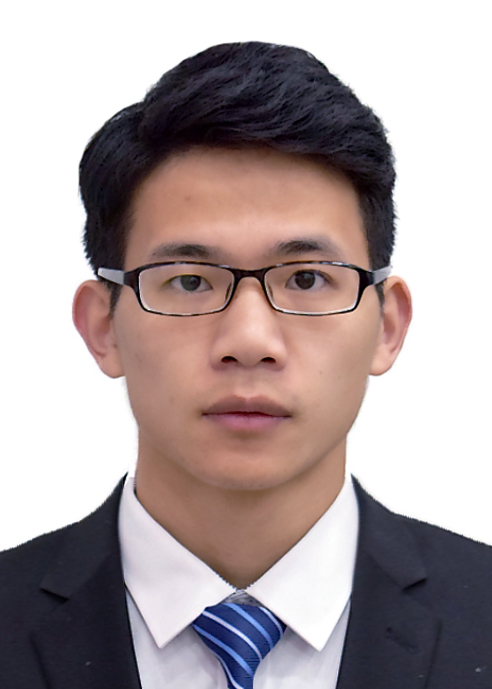}}]{Dayang Wang} is a visiting scholar at Department of Data Science, City University of Hong Kong. He earned his Ph.D. in Computer Engineering from the University of Massachusetts Lowell, where he was supervised by Prof. Hengyong Yu (IEEE Fellow). Prior to that, he received his bachelor's degree in Computer Science from Beijing Normal University. His current research primarily focuses on the development of machine learning models and their applications in medical imaging.
\end{IEEEbiography}

\vspace{-0.9cm}
\begin{IEEEbiography}[{\includegraphics[width=1in,height=1.25in,clip,keepaspectratio]{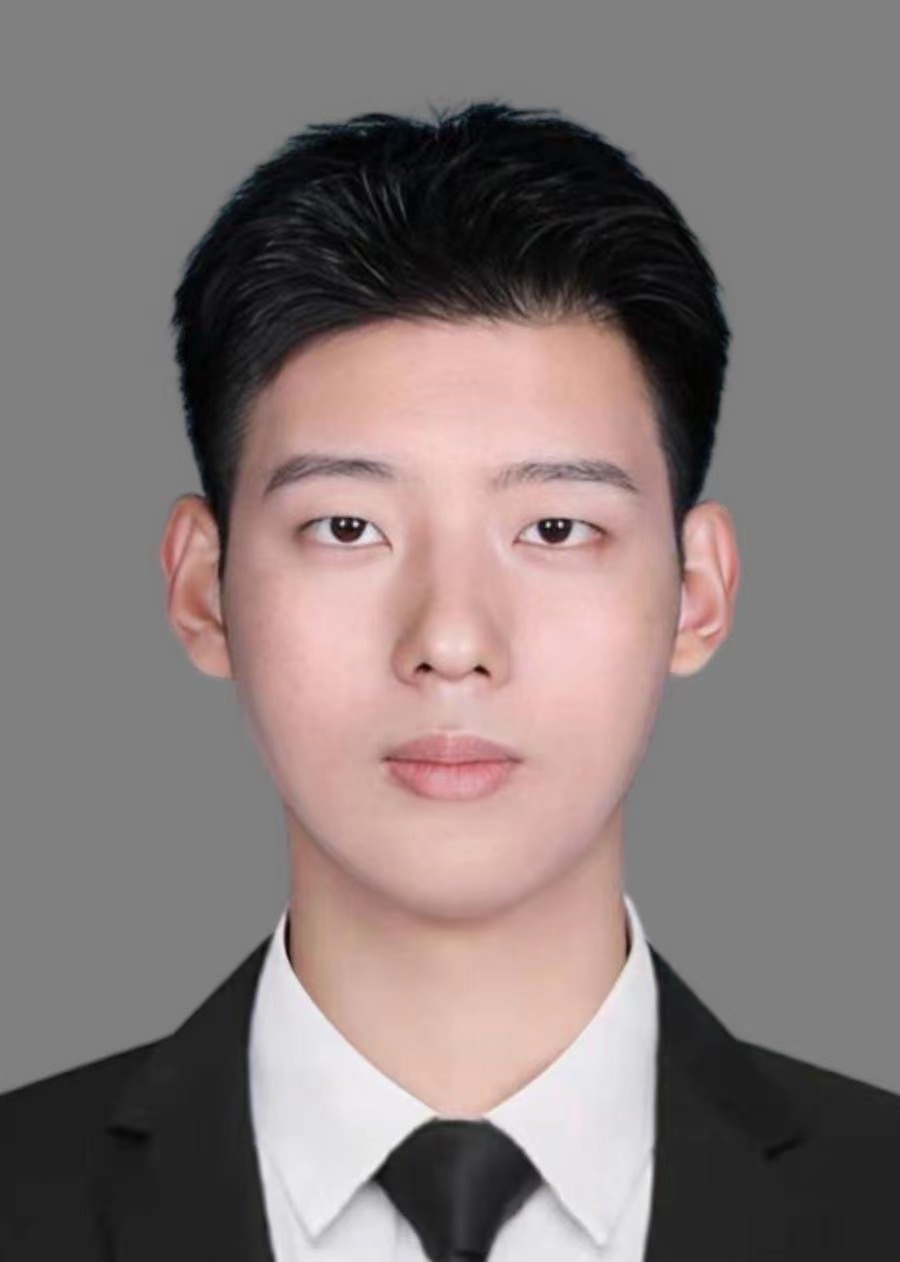}}]{Jingbo Zhang} is a Master’s student at Northeastern University, China. Prior to this, he also finished his undergraduate study in Northeastern University in 2021. In 2024, he was invited to work as a Research Assistant at the Department of Mathematics, The Chinese University of Hong Kong. Now, he is a remote student intern at Department of Data Science, City University of Hong Kong. His research interest is model compression.
\end{IEEEbiography}


\vspace{-0.9cm}
\begin{IEEEbiography}[{\includegraphics[width=1in,height=1.25in,clip,keepaspectratio]{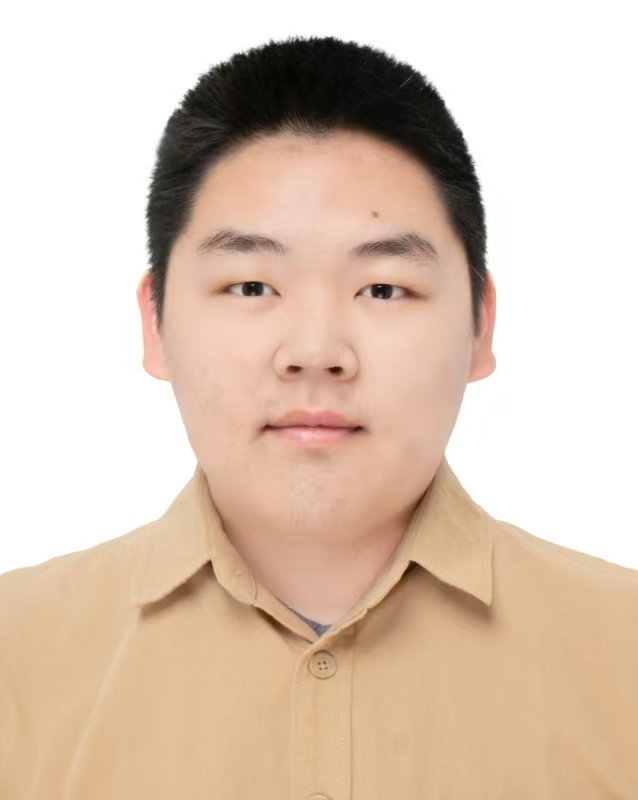}}]{Zelin Dong} is a Ph.D. student at the Department of Mathematics at the Chinese University of Hong Kong, having completed his undergraduate degree there in 2025. As a mentee in Prof. Fenglei Fan’s team (2022-2024), he conducted research in discrete computational geometry and received the UROP Gold Award in 2023. His research interests include partial differential equations and data science.
\end{IEEEbiography}

\vspace{-0.9cm}
\begin{IEEEbiography}[{\includegraphics[width=1in,height=1.25in,clip,keepaspectratio]{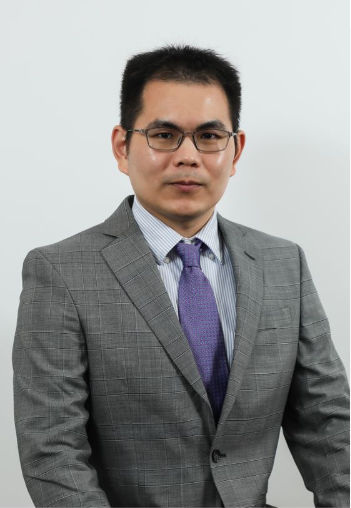}}]{Shijun Zhang} is an Assistant Professor at The Hong Kong Polytechnic University (PolyU). He graduated with a bachelor’s degree in mathematics from Wuhan University in 2016. Subsequently, he pursued his doctoral degree at National University of Singapore under the supervision of Professor Zuowei Shen and Professor Haizhao Yang, and successfully obtained his PhD in 2021. He used to work as a postdoctoral researcher at Duke University under the guidance of Professor Jianfeng Lu and Professor Hongkai Zhao. His primary research interest is in the (approximation) error analysis of neural networks.
\end{IEEEbiography}

\vspace{-0.9cm}
\begin{IEEEbiography}[{\includegraphics[width=1in,height=1.5in,clip,keepaspectratio]{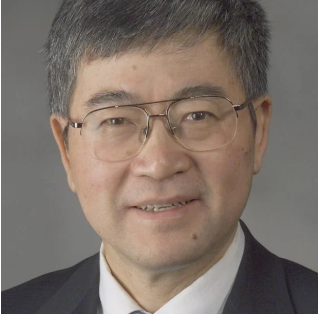}}]{Ge Wang} (Fellow of IEEE, SPIE, AAPM, OSA, AIMBE, AAAS, and NAI) is the Clark \& Crossan Chair Professor and Director of Biomedical Imaging Center at Rensselaer Polytechnic Institute (Troy, New York, USA). He pioneered the spiral cone-beam CT method in 1991 and published the first perspective on deep learning-based tomographic imaging in 2016. His honors include the IEEE EMBS Career Achievement Award, SPIE Meinel Technology Award, Sigma Xi Chubb Award for Innovation, RPI Wiley Distinguished Faculty Award, IEEE R1 Outstanding Teaching Award, and IEEE NPSS/NMISC Hoffman Medical Imaging Scientist Award.
\end{IEEEbiography}

\vspace{-0.9cm}
\begin{IEEEbiography}[{\includegraphics[width=1in,height=1.25in,clip,keepaspectratio]{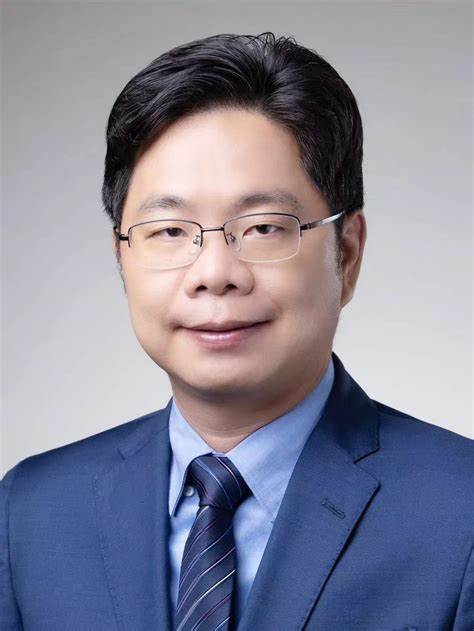}}]{Tieyong Zeng} received the B.S. degree from Peking
University, Beijing, China, in 2000, the M.S degree from Ecole Polytechnique, Palaiseau, France, in 2004, and the Ph.D. degree from the University
of Paris XIII, Paris, France, in 2007. He is
currently a Professor in the Department of Mathematics, The Chinese University of Hong Kong, Hong Kong. His research interests are image processing, machine learning, and
scientific computing.
\end{IEEEbiography}

\end{document}